\documentclass[twoside,11pt]{article} 
\usepackage[margin=1.0in]{geometry}

\usepackage{amssymb,amsmath,amsthm}
\usepackage{mathtools}
\usepackage{bm,lscape}
\usepackage{bbm}
\usepackage{mathrsfs,dsfont}
\usepackage{wasysym}
\RequirePackage[
    colorlinks=true,
    linkcolor=blue,
    urlcolor=blue,
    citecolor=blue]{hyperref}
\RequirePackage{natbib}
\usepackage[]{algorithm2e}
\usepackage[utf8]{inputenc}
\usepackage[english]{babel}

\usepackage{graphicx}
\usepackage{color}
\usepackage{rotating}
\usepackage{authblk}
\usepackage{appendix}

\allowdisplaybreaks

\def\bSig\mathbf{\Sigma}

	\newcommand{\fu}{f_i^{(1)}(\mathbf X)}
	
	\newcommand{\fl}{f_i^{(l)}(\mathbf X)}

\newcommand{\ddr}{\mathrm{d}}

\usepackage{booktabs}

\newtheorem{theorem}{Theorem}

\newtheorem{lem}[theorem]{Lemma}

\providecommand{\keywords}[1]
{
  \small	
  \textbf{\textit{Keywords:}} #1
}

\begin{document}

\title{Deep Stable neural networks: large-width asymptotics and convergence rates}


\author[1]{Stefano Favaro\thanks{stefano.favaro@unito.it}}
\author[2]{Sandra Fortini\thanks{sandra.fortini@unibocconi.it}}
\author[3]{Stefano Peluchetti\thanks{speluchetti@cogent.co.jp}}
\affil[1]{\small{Department of Economics and Statistics, University of Torino and Collegio Carlo Alberto, Italy}}
\affil[2]{\small{Department of Decision Sciences, Bocconi University, Italy}}
\affil[3]{\small{Cogent Labs, Tokyo, Japan}}

\maketitle

\begin{abstract}
In modern deep learning, there is a recent and growing literature on the interplay between large-width asymptotic properties of deep Gaussian neural networks (NNs), i.e. deep NNs with Gaussian-distributed weights, and Gaussian stochastic processes (SPs). Such an interplay has proved to be critical in Bayesian inference under Gaussian SP priors, kernel regression for infinitely wide deep NNs trained via gradient descent, and information propagation within infinitely wide NNs. Motivated by empirical analyses that show the potential of replacing Gaussian distributions with Stable distributions for the NN's weights, in this paper we present a rigorous analysis of the large-width asymptotic behaviour of (fully connected) feed-forward deep Stable NNs, i.e. deep NNs with Stable-distributed weights. We show that as the width goes to infinity jointly over the NN's layers, i.e. the ``joint growth" setting, a rescaled deep Stable NN converges weakly to a Stable SP whose distribution is characterized recursively through the NN's layers. Because of the non-triangular structure of the NN, this is a non-standard asymptotic problem, to which we propose an inductive approach of independent interest. Then, we establish sup-norm convergence rates of the rescaled deep Stable NN to the Stable SP, under the ``joint growth" and a ``sequential growth" of the width over the NN's layers. Such a result provides the difference between the ``joint growth" and the ``sequential growth" settings, showing that the former leads to a slower rate than the latter, depending on the depth of the layer and the number of inputs of the NN. Our work extends some recent results on infinitely wide limits for deep Gaussian NNs to the more general deep Stable NNs, providing the first result on convergence rates in the ``joint growth" setting.
\end{abstract}



\section{Introduction}\label{sec:intro}

Modern neural networks (NNs) feature a large number of layers (depth) and units per layer (width), and they have achieved a remarkable performance across numerous domains of practical interest \citep{Lec(15)}. In such a context, there is a recent and growing literature that investigates the interplay between the large-width asymptotic behaviour of deep Gaussian NNs, i.e. deep NNs with Gaussian-distributed weights, and Gaussian stochastic processes (SPs). See \cite{Nea(96),Wil(97),Der(06),Haz(15),Gar(18),Lee(18),Mat(18),Nov(18),Ant(19),Aro(19),Yan(19),Yan(19a),Ait(20),And(20),Eld(21),Klu(21),Bas(22)}, and references therein, for a comprehensive account on large-width asymptotic properties of deep Gaussian NNs, and generalizations thereof. Intuitively, the prototypical interplay between deep Gaussian NNs and Gaussian SPs may be stated as follows: as the NN's width goes to infinity jointly over the NN's layers, a suitable rescaled deep Gaussian NN converges weakly to a Gaussian SP whose characteristic (covariance) kernel is defined recursively through the NN's layers. To be more rigorous, we consider the popular class of (fully connected) feed-forward Gaussian NNs with depth $D\geq1$, width $n\geq1$ and $k\geq1$ input-signals of dimension $I\geq1$, though analogous results hold true for more general architectures, such as the popular convolutional NNs. In particular, we denote by $\text{N}(\mu,\sigma^{2})$ a Gaussian distribution with mean $\mu\in\mathbb{R}$ and variance $\sigma^{2}\in\mathbb{R}_{+}$, and by $\text{N}_{k}(\mu,\Sigma)$ a $k$-dimensional Gaussian distribution with mean vector $\mu\in\mathbb{R}^{k}$ and covariance matrix $\Sigma\in\mathbb{R}^{k}\times\mathbb{R}^{k}$. In the following theorem we recall the main result of \cite{Mat(18)}, which deals with the infinitely wide limit of a (fully connected) feed-forward deep Gaussian NN. We refer to \cite{Yan(19)} and \cite{Yan(19a)} for analogous results under more general architectures, e.g. convolutional NNs and their generalizations, and more general classes of activation functions.

\begin{theorem}[Deep Gaussian NNs \cite{Mat(18)}]\label{teo0}
For any $I\geq1$ and $k\geq1$ let $\mathbf{X}$ be a $I\times k$ (input-signal) matrix, with $\mathbf{x}_{j}$ being the $j$-th row vector of $\mathbf{X}$, and for any $D\geq1$ and $n\geq1$ let: i) $(\mathbf{W}^{(1)},\ldots,\mathbf{W}^{(D)})$ be a collection of i.i.d. random (weight) matrices, such that $\mathbf{W}^{(1)}=(w^{(1)}_{i,j})_{1\leq i\leq n,\,1\leq j\leq I}$ and $\mathbf{W}^{(l)}=(w^{(l)}_{i,j})_{1\leq i\leq n,\,1\leq j\leq n}$ for $2\leq l\leq D$, where the $w^{(l)}_{i,j}$'s are i.i.d. as $\text{N}(0,\sigma^{2}_{w})$ for $l=1,\ldots,D$; ii) $(\mathbf{b}^{(1)},\ldots,\mathbf{b}^{(D)})$ be a collection of i.i.d. random (bias) vectors, such that $\mathbf{b}^{(l)}=(b^{(l)}_{1},\ldots,b^{(l)}_{n})$ where the $b^{(l)}_{i}$'s are i.i.d. as $\text{N}(0,\sigma^{2}_{b})$ for $l=1,\ldots,D$; iii) $(\mathbf{W}^{(1)},\ldots,\mathbf{W}^{(D)})$ be independent of $(\mathbf{b}^{(1)},\ldots,\mathbf{b}^{(D)})$. Moreover, for some $a,b\geq0$, let $\phi:\mathbb{R}\rightarrow\mathbb{R}$ be a continuous activation function such that
\begin{equation}\label{eq:le1}
|\phi(s)|\leq a+b|s|
\end{equation}
for every $s\in\mathbb{R}$, and consider the NN $(f_{i}^{(l)}(\mathbf{X},n))_{1\leq i\leq n,1\leq l\leq D}$ of depth $D$ and width $n$ defined as follows
\begin{displaymath}
f_{i}^{(1)}(\mathbf{X})=\sum_{j=1}^{I}w_{i,j}^{(1)}\mathbf{x}_{j}+b_{i}^{(1)}\mathbf{1}^{T}
\end{displaymath}
and
\begin{displaymath}
f_{i}^{(l)}(\mathbf{X},n)=\frac{1}{\sqrt{n}}\sum_{j=1}^{n}w_{i,j}^{(l)}(\phi\circ f_{j}^{(l-1)}(\mathbf{X},n))+b_{i}^{(l)}\mathbf{1}^{T},
\end{displaymath}
with $f_{i}^{(1)}(\mathbf{X},n):=f_{i}^{(1)}(\mathbf{X})$, where $\mathbf{1}$ is the $k$-dimensional unit (column) vector, and $\circ$ denotes the element-wise application.
For any $l=1,\ldots,D$, if $(f_i^{(l)}(\mathbf X,n))_{i\geq 1}$ is the sequence obtained by extending $(\mathbf W^{(1)},\dots,\mathbf W^{(D)})$ and
$(\mathbf b^{(1)},\dots,\mathbf b^{(D)})$ to infinite i.i.d. arrays, then as $n\rightarrow+\infty$ jointly over the first $l$ layers
\begin{displaymath}
(f_{i}^{(l)}(\mathbf{X},n))_{i\geq1}\stackrel{\text{w}}{\longrightarrow}(f_{i}^{(l)}(\mathbf{X}))_{i\geq1},
\end{displaymath}
where $(f_{i}^{(l)}(\mathbf{X}))_{i\geq1}$ is distributed as $\otimes_{i\geq1}\text{N}_{k}(\mathbf{0},\Sigma^{(l)})$, with the covariance matrix $\Sigma^{(l)}$ having the $(u,v)$-th entry
\begin{displaymath}
\Sigma_{u,v}^{(1)}=\sigma_b^2 + \sigma_w^2\langle \mathbf{x}_{u}, \mathbf{x}_{v} \rangle
\end{displaymath}
and
\begin{displaymath}
\Sigma_{u,v}^{(l)}=\sigma_b^2 + \sigma_w^2 \mathbb{E}_{\tiny{(f,g)\sim\text{N}_{2}\left(\begin{pmatrix}0 \\ 0\end{pmatrix}, \begin{pmatrix}\Sigma_{u,u}^{(l-1)}&\Sigma_{u,v}^{(l-1)}\\\Sigma_{v,u}^{(l-1)} & \Sigma_{v,v}^{(l-1)} \end{pmatrix}\right)}}[\phi(f)\phi(g)].
\end{displaymath}
Then, the limiting SP $(f_{i}^{(l)}(\mathbf{X}))_{i\geq1}$, as a process indexed by $\mathbf{X}$, is a Gaussian SP with parameter or kernel $\Sigma$.
\end{theorem}

Theorem \ref{teo0} generalizes an early result of \cite{Nea(96)}, which provides the infinitely wide limit under the assumption that the width $n$ goes to infinity sequentially over the NN's layers, i.e. $n\rightarrow+\infty$ one layer at a time. Under the ``sequential growth" setting, the study of large-width asymptotics reduces to an application of Lindeberg-L\'evy central limit theorem. Instead, assuming a ``joint growth" of the width over the NN's layers, i.e. $n\rightarrow+\infty$ simultaneously over the first $l\geq1$ layers, makes Theorem \ref{teo0} a non-standard asymptotic problem, whose solution is obtained by adapting a central limit theorem for triangular arrays to the non-triangular structure of the NN \citep{Blu(58)}. Theorem \ref{teo0} has been exploited in many directions: i) Bayesian inference for Gaussian SPs arising from infinitely wide NNs \citep{Lee(18),Gar(18)}; ii) kernel regression for infinitely wide NNs trained with gradient descent through the neural tangent kernel \citep{Jac(18), Lee(19), Aro(19)}; iii) statistical analysis of infinitely wide NNs as functions of the depth via information propagation \citep{Poo(16), Sch(17),Hay(19)}. It has been shown a substantial gap, in terms of empirical performance, between deep NNs and their corresponding infinitely wide Gaussian SPs, at least on some benchmarks applications. Such a gap is prominent in the case of convolutional NNs, while for the fully-connected NNs object of this study infinitely wide Gaussian SPs prove competitive \citep{Lee(20)}. Moreover, it is known to be a difficult task to avoid undesirable empirical properties arising in deep NNs. Given that, there is an increasing interest in extending the class of Gaussian SPs arising as infinitely wide limits of deep NNs, as a way forward to reduce such a performance gap and to avoid, or slow down, common pathological behaviors.

\subsection{Our contributions}

In this paper, we study SPs arising as infinitely wide limits of deep Stable NNs, i.e. deep NNs with Stable-distributed weights \citep{Sam(94)}. Stable distributions form a broad class of heavy tails or infinite variance distributions indexed by a parameter $\alpha\in(0,2]$, and they are arguably the most natural generalization of the Gaussian distribution. The works \cite{Nea(96)} and \cite{Der(06)} first discussed the use of the Stable distribution for initializing deep NNs, leaving as an open problem the rigorous study of large-width asymptotic properties of deep Stable NNs. Empirical analyses in \cite{Nea(96)} show the following large-width phenomenon: while the contribution of Gaussian weights vanishes in the infinitely wide limit, Stable weights retain a non-negligible contribution, allowing them to represent ``hidden features”. This phenomenon suggests a more flexible behaviour of NN's weights with heavy tails, which results in infinitely wide SPs with a different behaviour than Gaussian SPs. In a classification setting, deep NNs trained with stochastic gradient descent result in heavy-tailed distributions for the weights, as a consequence of the training dynamics \citep{favaro2020,For(20),Hod(21)}. In such a setting, empirical analyses in \cite{For(20)} show that the use of NN's weights that are Stable-distributed leads to a higher classification accuracy, as it results in different path properties. See Figure~\ref{fig:shapes} for (function) samples realized by wide fully-connected NNs whose weights are distributed as Stable distributions with decreasing $\alpha$, i.e. distributions with increasingly heavy tails. Recently, \cite{li2021bayesian} investigated the use of deep Stable NNs for image inverse problems when images contain sharp edges. Within this setting, the abrupt jumps allowed by the NN function mapping for lower values of $\alpha$ result in a better matching prior for the problem of interest, and in superior performance in terms of inference.

\begin{figure}[htp!]
\centering
\includegraphics[width=0.27\textwidth, trim = 40 40 40 40]{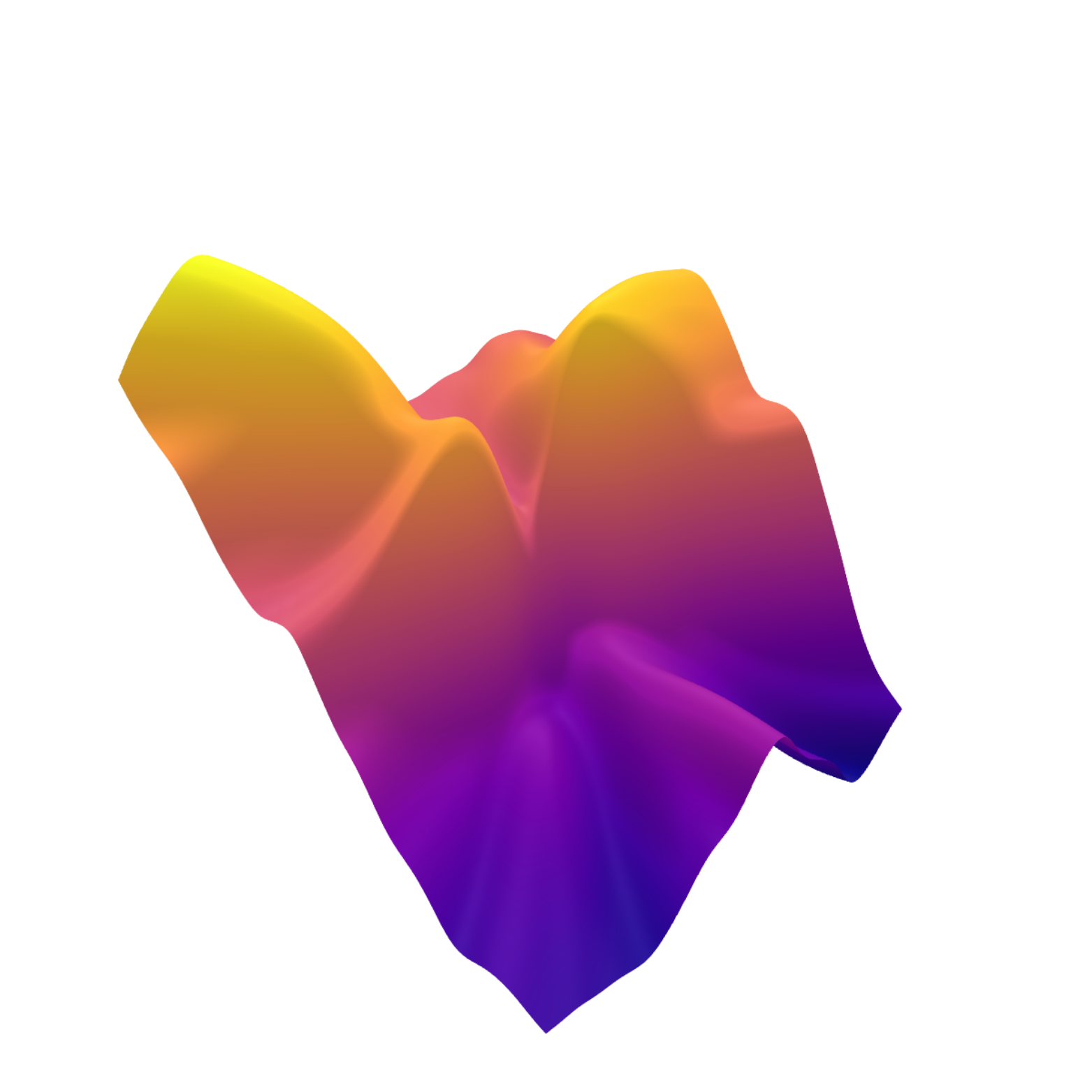}
\includegraphics[width=0.27\textwidth, trim = 40 40 40 40]{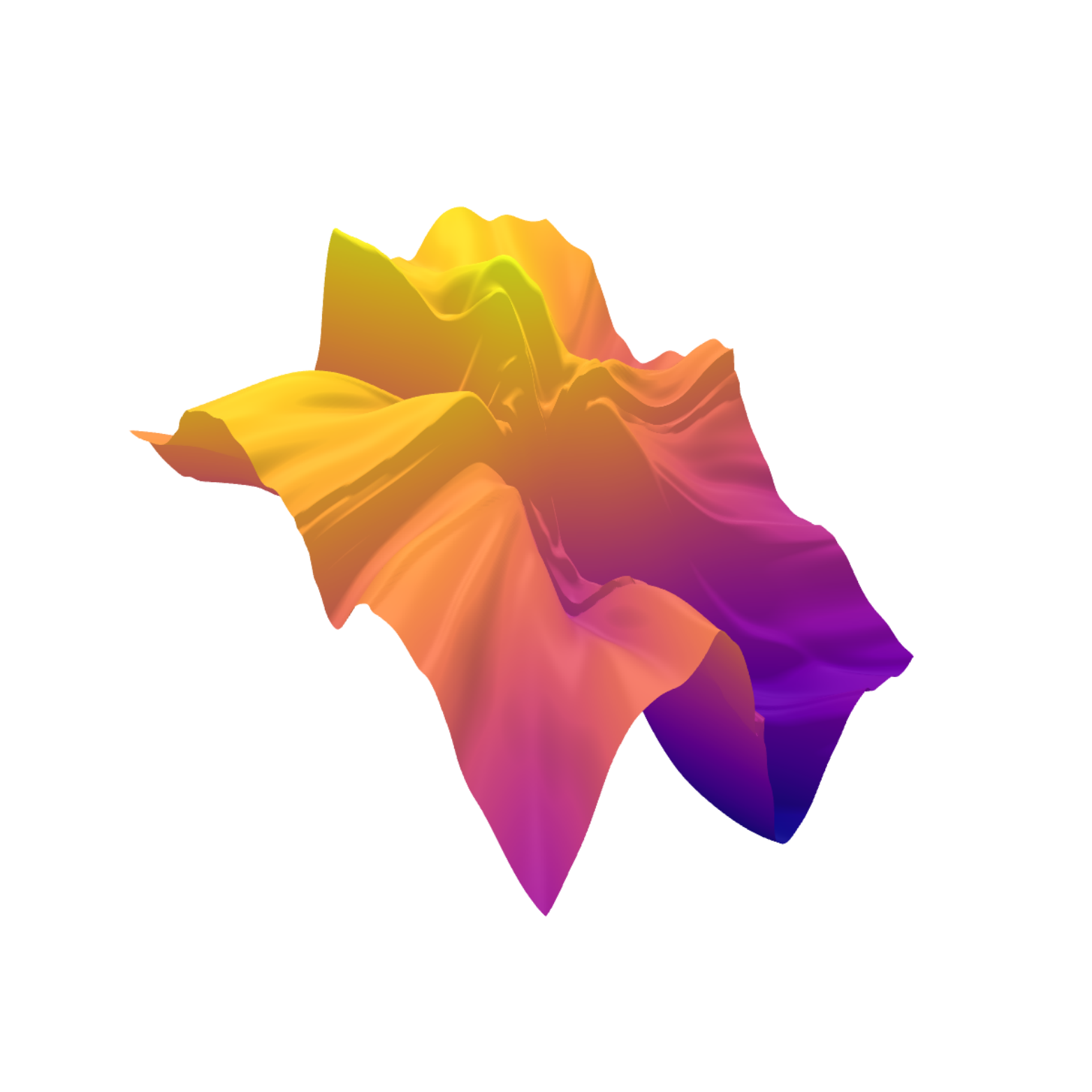}\\
\includegraphics[width=0.27 \textwidth, trim = 40 40 40 40]{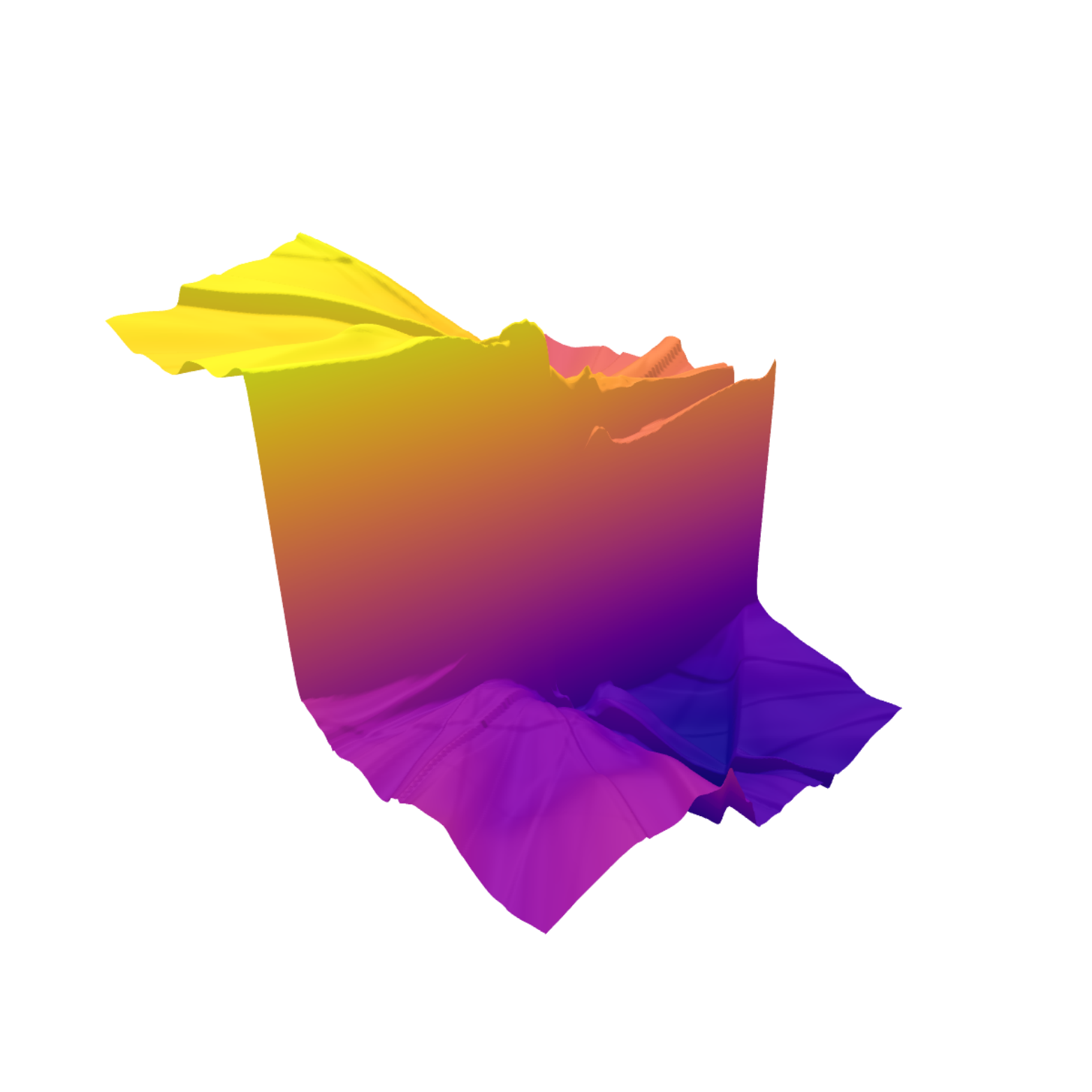}
\includegraphics[width=0.27\textwidth, trim = 40 40 40 40]{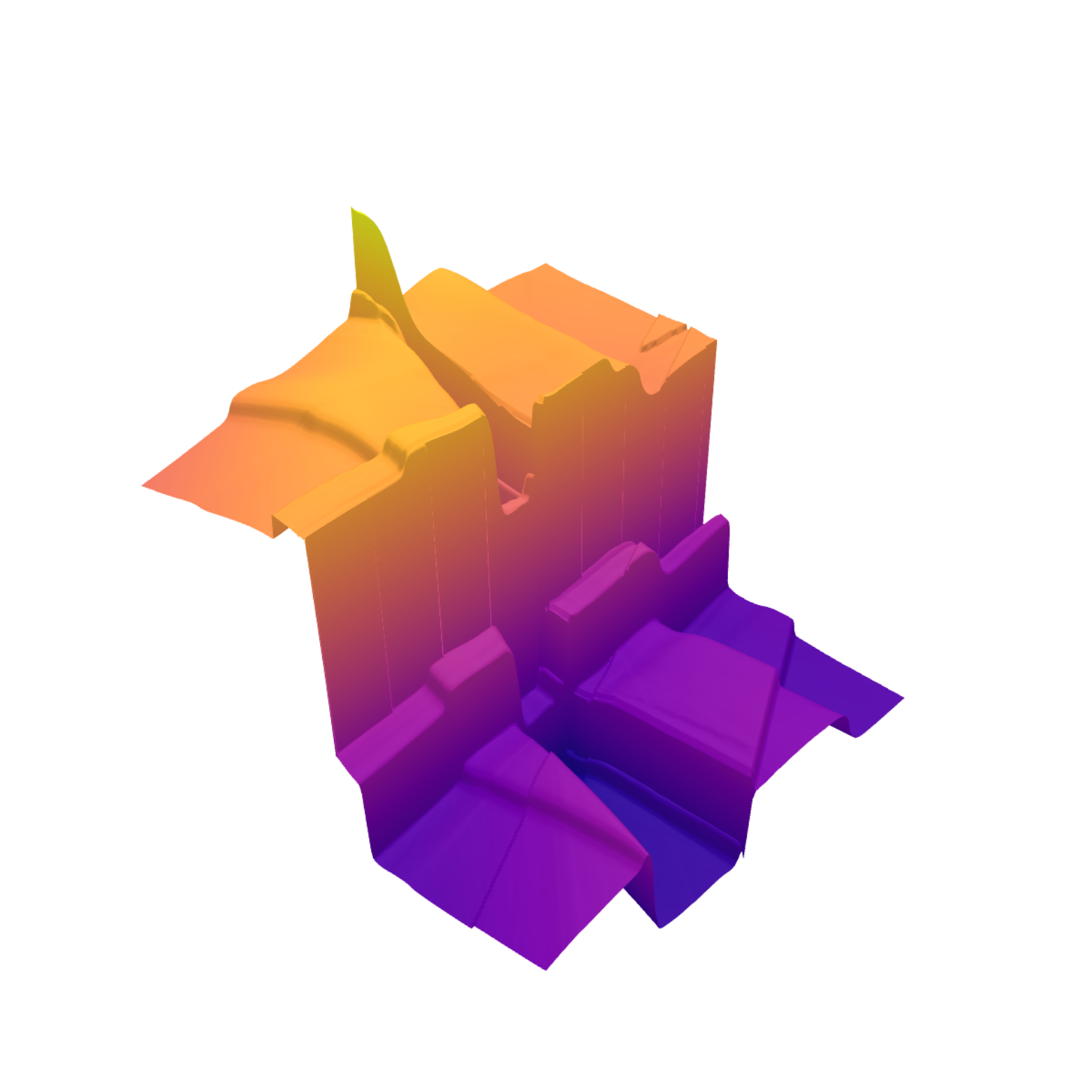}
\caption{Samples of a fully-connected Stable NN mapping $[0,1]^2$ to $\mathbb{R}$, with a $\tanh$ activation function and $D=2$ hidden layers of width $n=1024$, for different values of the parameter $\alpha$: i) $\alpha=2.0$ (Gaussian distribution) top-left; ii) $\alpha=1.5$ top-right; iii) $\alpha=1.0$ (Cauchy distribution) bottom-left; iv) $\alpha=0.5$ (L\'evy distribution) bottom-right.}\label{fig:shapes}
\end{figure}

Motivated by the recent interest in deep Stable NNs, we present a rigorous analysis of the large-width asymptotic behaviour of (fully connected) feed-forward deep Stable NNs, with depth $D\geq1$, width $n\geq1$ and $k\geq1$ input-signals of dimension $I\geq1$. We denote by $\text{St}(\alpha,\sigma)$ the symmetric centered Stable distribution with stability parameter $\alpha\in(0,2)$ and scale parameter $\sigma$, and by $\text{St}_{k}(\alpha,\Gamma)$ the symmetric centered $k$-dimensional Stable distribution with stability parameter $\alpha\in(0,2)$ and scale (finite) spectral measure $\Gamma$ on the unit sphere $\mathbb{S}^{k-1}$ in $\mathbb{R}^{k}$. We refer to \cite[Chapter 1,2]{Sam(94)} for a detailed account of Stable distributions. The case $\alpha=2$ is the Gaussian distribution, which is excluded from our analysis. The next theorem states our first main result, which extends Theorem \ref{teo0} to deep Stable NNs. A preliminary version of the theorem appeared in \citep{favaro2020}, though with a non-rigorous statement and proof.

\begin{theorem}[Deep Stable NNs]\label{teo1}
For any $I\geq1$ and $k\geq1$ let $\mathbf{X}$ be a $I\times k$ (input-signal) matrix, with $\mathbf{x}_{j}$ being the $j$-th row of $\mathbf{X}$, and for any $D\geq1$ and $n\geq1$ let: i) $(\mathbf{W}^{(1)},\ldots,\mathbf{W}^{(D)})$ be a collection of i.i.d. random (weight) matrices, such that $\mathbf{W}^{(1)}=(w^{(1)}_{i,j})_{1\leq i\leq n,\,1\leq j\leq I}$ and $\mathbf{W}^{(l)}=(w^{(l)}_{i,j})_{1\leq i\leq n,\,1\leq j\leq n}$ for $2\leq l\leq D$, where the $w^{(l)}_{i,j}$'s are i.i.d. as $\text{St}(\alpha,\sigma_{w})$ for $l=1,\ldots,D$; ii) $(\mathbf{b}^{(1)},\ldots,\mathbf{b}^{(D)})$ be a collection of i.i.d. random (bias) vectors, such that $\mathbf{b}^{(l)}=(b^{(l)}_{1},\ldots,b^{(l)}_{n})$ where the $b^{(l)}_{i}$'s are i.i.d. as $\text{St}(\alpha,\sigma_{b})$ for $l=1,\ldots,D$; iii) $(\mathbf{W}^{(1)},\ldots,\mathbf{W}^{(D)})$ be independent of $(\mathbf{b}^{(1)},\ldots,\mathbf{b}^{(D)})$. Moreover, for some $a,b,\beta>0$, with $\beta<1$, let $\phi:\mathbb{R}\rightarrow\mathbb{R}$ be a continuous activation function such that
\begin{equation}\label{eq:le}
|\phi(s)|\leq a+b|s|^{\beta}
\end{equation}
for every $s\in\mathbb{R}$, and consider the NN $(f_{i}^{(l)}(\mathbf{X},n))_{1\leq i\leq n,1\leq l\leq D}$ of depth $D$ and width $n$ defined as follows
\begin{displaymath}
f_{i}^{(1)}(\mathbf{X})=\sum_{j=1}^{I}w_{i,j}^{(1)}\mathbf{x}_{j}+b_{i}^{(1)}\mathbf{1}^{T}
\end{displaymath}
and
\begin{displaymath}
f_{i}^{(l)}(\mathbf{X},n)=\frac{1}{n^{1/\alpha}}\sum_{j=1}^{n}w_{i,j}^{(l)}(\phi\circ f_{j}^{(l-1)}(\mathbf{X},n))+b_{i}^{(l)}\mathbf{1}^{T}
\end{displaymath}
with $f_{i}^{(1)}(\mathbf{X},n):=f_{i}^{(1)}(\mathbf{X})$, where $\mathbf{1}$ is the $k$-dimensional unit (column) vector, and $\circ$ denotes the element-wise application.
For any $l=1,\ldots,D$, if $(f_i^{(l)}(\mathbf X,n))_{i\geq 1}$ is the sequence obtained by extending $(\mathbf W^{(1)},\dots,\mathbf W^{(D)})$ and
$(\mathbf b^{(1)},\dots,\mathbf b^{(D)})$ to infinite i.i.d. arrays, then as $n\rightarrow+\infty$ jointly over the first $l$ layers
\begin{displaymath}
(f_{i}^{(l)}(\mathbf{X},n))_{i\geq1}\stackrel{\text{w}}{\longrightarrow}(f_{i}^{(l)}(\mathbf{X}))_{i\geq1},
\end{displaymath}
where $(f_{i}^{(l)}(\mathbf{X}))_{i\geq1}$ is distributed as $\otimes_{i\geq1}\text{St}_{k}(\alpha,\Gamma^{(l)})$, with $\alpha\in(0,2)$, with the spectral measure $\Gamma^{(l)}$ being defined as
\begin{displaymath}
\Gamma^{(1)}=||\sigma_{b}\mathbf{1}^{T}||^{\alpha}\zeta_{\frac{\mathbf{1}^{T}}{||\mathbf{1}^{T}||}}+\sigma_{w}^{\alpha}\sum_{j=1}^{I}||\mathbf{x}_{j}||^{\alpha}\zeta_{\frac{\mathbf{x}_{j}}{||\mathbf{x}_{j}||}}
\end{displaymath}
and
\begin{equation}\label{lim_spect}
\Gamma^{(l)}=||\sigma_{b}\mathbf{1}^{T}||^{\alpha}\zeta_{\frac{\mathbf{1}^{T}}{||\mathbf{1}^{T}||}}+\int||\sigma_{w}(\phi\circ{f})||^{\alpha}\zeta_{\frac{\phi\circ{f}}{||\phi\circ{f}||}}q^{(l-1)}(\text{d}{f}),
\end{equation}
with $||\cdot ||$ being the Euclidean norm in $\mathbb R^k$, where $\zeta_{h/||h||}=2^{-1}(\delta_{h/||h||} + \delta_{-h/||h||})I(||h|| > 0)$ with  $\delta$ being the Dirac measure and $I$ being the indicator function, where  $q^{(l-1)}$ denotes the distribution of $f_{i}^{(l-1)}(\mathbf{X})$. Then, the limiting SP $(f_{i}^{(l)}(\mathbf{X}))_{i\geq1}$, as a process indexed by $\mathbf{X}$, is a Stable SP with parameter $(\alpha,\Gamma)$.
\end{theorem}

Critical to Theorem \ref{teo1} is the assumption \eqref{eq:le}, which being stronger than \eqref{eq:le1} restricts the class of activation functions that lead to nontrivial infinitely wide limits. Such a restricted class, however, still includes popular activation functions, e.g. logistic, hyperbolic tangent and Gaussian. See \cite{Bor(22)} and \cite{favaro2022} for infinitely wide limits of shallow Stable NNs with linear activation functions. As for Theorem \ref{teo0}, the non-triangular structure of the NN and the ``joint growth" of the width over the NN's layers make Theorem \ref{teo1} a non-standard asymptotic problem, with the additional challenge of dealing with heavy tails distributions. The proof of Theorem \ref{teo1} relies on the exchangeability of $(f_{i}^{(l)}(\mathbf{X},n))_{i\geq1}$ and, through de Finetti representation theorem, it exploits an inductive argument for the de Finetti measures over the NN's layers; this is a novel  approach of independent interest. Under the ``sequential growth" of the width over the NN's layers, the proof of Theorem \ref{teo1} reduces to an application of a generalized central limit theorem \cite{Gne(54)}, which leads to the same limiting SP.  Consistency or compatibility of the finite-dimensional distributions of the limiting Stable SP is also proved. As a refinement of Theorem \ref{teo1}, our second main result establishes sup-norm convergence rates of the rescaled deep Stable NN $f_{i}^{(l)}(\mathbf{X},n)$ to the Stable SP, in both the ``joint growth" setting and the ``sequential growth" setting. In particular, such a result shows that the ``joint growth" leads to a slower rate than the ``sequential growth", depending on the depth of the layer and $k$. This is the first result on converge rates in the ``joint growth" setting, providing the difference between the ``sequential growth" and the ``joint growth" settings.

\subsection{Organization of the paper}

The paper is structured as follows. In Section \ref{sec2} we prove Theorem \ref{teo1} and show consistency or compatibility of the finite-dimensional distributions of the limiting Stable SP, whereas in Section \ref{sec3} we establish sup-norm convergence of the deep Stable NNs to the Stable SP under the ``joint growth" and the ``sequential growth" of the width over the NN's layers. Section \ref{sec4} contains a discussion of our results with respect to Bayesian inference, neural tangent kernel analysis via gradient descent and large-depth limits.


\section{Proof of Theorem \ref{teo1}}\label{sec2}
Random variables are defined on a probability space $(\Omega,\mathcal{G},\mathbb{P})$, and we denote expectation by $\mathbb{E}$; inequalities between conditional probabilities and between expectations must be interpreted as $\mathbb P$-a.s. For every $l\geq1$ and $n\geq1$ we denote by $\mathcal{G}_{n,l}$ the sigma algebra generated by $\{(f^{(l^{\prime})}_{i}(\mathbf{X},m))_{i\geq1}\text{ : }m\leq n\text{ and }l^{\prime}\leq l\}$, by $\mathcal{G}_{n,0}$ the trivial sigma algebra. Let $\mathbb{E}_{n,l}$ and $\mathbb{P}_{n,l}$ be the conditional expectation and the conditional distribution, respectively, given $\mathcal{G}_{n,l}$. For fixed $n$, $\mathcal G_{n,l'}\subset \mathcal G_{n,l}$ whenever $l'<l$. For fixed $l$, $(\mathcal G_{n,l})_{n\geq 0}$ is a filtration. We denote by $\mathcal G_{\infty,l}$ the limit sigma-algebra, that is the sigma-algebra generated by $\cup_{n\geq 0}\mathcal G_{n,l} $. The conditional expectation, given $\mathcal G_{\infty,l}$, is denoted by $\mathbb E_{\infty,l}$. If $S\sim\text{St}(\alpha,\sigma)$ then for $t>0$
\begin{displaymath}
\mathbb{E}(\exp\{\text{i}tS\})=\exp\left\{-\sigma^{\alpha}|t|^{\alpha}\right\}.
\end{displaymath}
If $\mathbf{S}$ is a $k$-dimensional (row) random vector such that $\mathbf{S}\sim\text{St}_{k}(\alpha,\Gamma)$, then for a $k$-dimensional (column) vector $\mathbf{t}$
\begin{displaymath}
\mathbb{E}(\exp\{\text{i}\mathbf{S}\mathbf{t}\})=\exp\left\{-\int_{\mathbb{S}^{k-1}}|\mathbf{s}\mathbf{t}|^{\alpha}\Gamma(\ddr\mathbf{s})\right\}.
\end{displaymath}
If $\mathbf{1}_{r}$ is a $k$-dimensional (column) vector with all zeroes except a $1$ at the $r$-th component, then the ($1$-dimensional) $r$-th element of $\mathbf{S}$ is distributed as an $\alpha$-Stable distribution with stability $\alpha\in(0,2)$ and scale
\begin{displaymath}
\sigma=\left(\int_{\mathbb{S}^{k-1}}| \mathbf{s}\mathbf{1}_{r}|^{\alpha}\Gamma(\ddr \mathbf{s})\right)^{1/\alpha}.
\end{displaymath}
Throughout this section, we most deal with $k$-dimensional  $\alpha$-Stable distributions with discrete spectral measure, that is $\Gamma(\cdot)=\sum_{1\leq i\leq n}\gamma_{i}\delta_{\mathbf{s}_{i}}$ with $n\in\mathbb{N}$, $\gamma_{i}\in\mathbb{R}$ and $\mathbf{s}_{i}\in\mathbb{S}^{k-1}$, for $i=1,\ldots,n$ \citep[Chapter 2]{Sam(94)}.

We start the proof by determining the distributions of the $k$-dimensional (row) random vectors $f_{i}^{(1)}(\mathbf{X})$ and $f_{i}^{(l)}(\mathbf{X},n)\,|\,\{f_{j}^{(l-1)}(\mathbf{X},n)\}_{j=1,\ldots,n}$ for $l=2,\ldots,D$. We denote by $f_{i,r}^{(l)}(\mathbf{X},n)$ the $r$-th component of $f_{i}^{(l)}(\mathbf{X},n)$, that is $f_{i,r}^{(l)}(\mathbf{X},n)=f_{i}^{(l)}(\mathbf{X},n)\mathbf{1}_{r}$. If $\mathbf{t}$ is a $k$-dimensional (column) vector, then
\begin{align*}
\varphi_{f_{i}^{(1)}(\mathbf{X})}(\mathbf{t}) & =\mathbb{E}[e^{\textrm{i}f_{i}^{(1)}(\mathbf{X})\mathbf{t}}]                                                                                                                                                                                                                                             \\
                                              & =\mathbb{E}\left[\exp\left\{\textrm{i}\left[\sum_{j=1}^{I}w_{i,j}^{(1)}\mathbf{x}_{j}+b_{i}^{(1)}\mathbf{1}^{T}\right]\mathbf{t}\right\}\right]                                                                                                                                                          \\
                                              & =\exp\left\{-\sigma^{\alpha}_{b}||\mathbf{1}^{T}||^{\alpha}\left|\frac{\mathbf{1}^{T}}{||\mathbf{1}^{T}||}\mathbf{t}\right|^{\alpha}\right\}\exp\left\{-\sigma_{w}^{\alpha}\sum_{j=1}^{I}||\mathbf{x}_{j}||^{\alpha}\left|\frac{\mathbf{x}_{j}}{||\mathbf{x}_{j}||}\mathbf{t}\right|^{\alpha}\right\}    \\
                                              & =\exp\left\{-\int_{\mathbb{S}^{k-1}}|\mathbf{s}\mathbf{t}|^{\alpha}\left(||\sigma_{b}\mathbf{1}^{T}||^{\alpha}\zeta_{\frac{\mathbf{1}^{T}}{||\mathbf{1}^{T}||}}+\sum_{j=1}^{I}||\sigma_{w}\mathbf{x}_{j}||^{\alpha}\zeta_{\frac{\mathbf{x}_{j}}{||\mathbf{x}_{j}||}}\right)(\text{d}\mathbf{s})\right\}.
\end{align*}
Therefore, $f_{i}^{(1)}(\mathbf{X})\sim\text{St}_{k}(\alpha,\Gamma^{(1)})$, where
\begin{equation}\label{spectral1}
\Gamma^{(1)}=||\sigma_{b}\mathbf{1}^{T}||^{\alpha}\zeta_{\frac{\mathbf{1}^{T}}{||\mathbf{1}^{T}||}}+\sum_{j=1}^{I}||\sigma_{w}\mathbf{x}_{j}||^{\alpha}\zeta_{\frac{\mathbf{x}_{j}}{||\mathbf{x}_{j}||}}.
\end{equation}
If $f_{i,r}^{(1)}(\mathbf{X})$ is the $r$-th component of the random vector $f_{i}^{(1)}(\mathbf{X})$, then $f_{i,r}^{(1)}(\mathbf{X})\sim \text{St}(\alpha,\sigma^{(1)}(r))$ with $\sigma^{(1)}(r)=\left(\int_{\mathbb{S}^{k-1}}|\mathbf{s}\mathbf{1}_{r}|^{\alpha}\Gamma^{(1)}(\text{d}\mathbf{s})\right)^{1/\alpha}$ \citep[Chapter 2]{Sam(94)}. Along similar lines, for each $l=2,\ldots,D$ we write
\begin{align*}
 & \varphi_{f_{i}^{(l)}(\mathbf{X},n)\,|\,\{f_{j}^{(l-1)}(\mathbf{X},n)\}_{j=1,\ldots,n}}(\mathbf{t})                                                                                                                                                                                                                                                                                              \\
 & \quad=\mathbb{E}[e^{\textrm{i}f_{i}^{(l)}(\mathbf{X},n)\mathbf{t}}\,|\,\{f_{j}^{(l-1)}(\mathbf{X},n)\}_{j=1,\ldots,n}]                                                                                                                                                                                                                                                                          \\
 & \quad=\mathbb{E}\left[\exp\left\{\textrm{i}\left[\frac{1}{n^{1/\alpha}}\sum_{j=1}^{n}w_{i,j}^{(l)}(\phi\circ f_{j}^{(l-1)}(\mathbf{X},n))+b_{i}^{(l)}\mathbf{1}^{T}\right]\mathbf{t}\right\}\,|\,\{f_{j}^{(l-1)}(\mathbf{X},n)\}_{j=1,\ldots,n}\right]                                                                                                                                          \\
 & \quad=\exp\left\{-\sigma^{\alpha}_{b}||\mathbf{1}^{T}||^{\alpha}\left|\frac{\mathbf{1}^{T}}{||\mathbf{1}^{T}||}\mathbf{t}\right|^{\alpha}\right\}\exp\left\{-\frac{\sigma^{\alpha}_{w}}{n}\sum_{j=1}^{n}||\phi\circ f_{j}^{(l-1)}(\mathbf{X},n)||^{\alpha}\,\left|\,\frac{\phi\circ f_{j}^{(l-1)}(\mathbf{X},n)}{||\phi\circ f_{j}^{(l-1)}(\mathbf{X},n)||}\mathbf{t}\right|^{\alpha}\right\}   \\
 & \quad=\exp\left\{-\int_{\mathbb{S}^{k-1}}|\mathbf{s}\mathbf{t}|^{\alpha}\left(||\sigma_{b}\mathbf{1}^{T}||^{\alpha}\zeta_{\frac{\mathbf{1}^{T}}{||\mathbf{1}^{T}||}}+\frac{1}{n}\sum_{j=1}^{n}||\sigma_{w}(\phi\circ f_{j}^{(l-1)}(\mathbf{X},n))||^{\alpha}\zeta_{\frac{\phi\circ f_{j}^{(l-1)}(\mathbf{X},n)}{||\phi\circ f_{j}^{(l-1)}(\mathbf{X},n)||}}\right)(\text{d}\mathbf{s})\right\}.
\end{align*}
Therefore, $f_{i}^{(l)}(\mathbf{X},n)\,|\,\{f_{j}^{(l-1)}(\mathbf{X},n)\}_{j=1,\ldots,n}\sim\text{St}_{k}(\alpha,\Gamma_n^{(l)})$, where
\begin{equation}\label{spectral2}
\Gamma_n^{(l)}=||\sigma_{b}\mathbf{1}^{T}||^{\alpha}\zeta_{\frac{\mathbf{1}^{T}}{||\mathbf{1}^{T}||}}+\frac{1}{n}\sum_{j=1}^{n}||\sigma_{w}(\phi\circ f_{j}^{(l-1)}(\mathbf{X},n))||^{\alpha}\zeta_{\frac{\phi\circ f_{j}^{(l-1)}(\mathbf{X},n)}{||\phi\circ f_{j}^{(l-1)}(\mathbf{X},n)||}}.
\end{equation}
Then, $f_{i,r}^{(l)}(\mathbf{X},n)\,|\,\{f_{j}^{(l-1)}(\mathbf{X},n)\}_{j=1,\ldots,n}\sim \text{St}(\alpha,\sigma_n^{(l)}(r))$ with $\sigma_n^{(l)}(r)=\left(\int_{\mathbb{S}^{k-1}}|\mathbf{s}\mathbf{1}_{r}|^{\alpha}\Gamma_n^{(l)}(\text{d}\mathbf{s})\right)^{1/\alpha}$ \citep[Chapter 2]{Sam(94)}.

Hereafter, we establish the infinitely wide limit of the $(f^{(l)}_{i}(\mathbf{X},n))_{i\geq1}$. The proof exploits the exchangeability of $(f_{i}^{(l)}(\mathbf{X},n))_{i\geq1}$ and an inductive argument, over the NN's layers, for the directing (de Finetti) random probability measure of $(f_{i}^{(l)}(\mathbf{X},n))_{i\geq1}$. For $l=1$, by definition, the random vectors $({f}_i^{(1)}(\mathbf{X},n))_{i\geq1}$ are i.i.d. according to the probability measure $p_{n}^{(1)}=St_k(\alpha,\Gamma^{(1)})$, where $\Gamma^{(1)}$ is defined in \eqref{spectral1}. Then, as $n\rightarrow+\infty$, $p_n^{(1)}$ converges weakly to $q^{(1)}=St_k(\alpha,\Gamma^{(1)})$. For $l>1$, the random vectors
$(f_i^{(l)}(\mathbf X,n))_{i\geq 1}$ are conditionally i.i.d. given $(f_j^{(l-1)}(\mathbf X,n))_{j=1,\dots,n}$, with (random) probability measure $p_n^{(l)}=St_k(\alpha,\Gamma_n^{(l)})$, where  $\Gamma_n^{(l)}$ is defined in \eqref{spectral2}. Given $(f_j^{(l-1)}(\mathbf X,n))_{j=1,\dots,n}$, the sequence $(f_i^{(l)}(\mathbf X,n))_{i\geq 1}$ is conditionally independent of
$\{(f_{j}^{(l')}(\mathbf X,m))_{j\geq 1}\text{ : }m\leq n,l'\leq l-1,(j,l',m)\not\in \{1,\dots,n\}\times\{l-1\}\times\{n\}\}$. It follows that $(f_i^{(l)}(\mathbf X,n))_{i\geq 1}$ are conditionally i.i.d., given $\mathcal G_{n,l-1}$, with (random) probability measure $p_n^{(l)}$. Before stating the induction hypothesis, we give a preliminary result.

\begin{lem}\label{lem1}
Let $\epsilon>0$ be such that $\alpha+\epsilon<2$ and $(\alpha+\epsilon)\beta<\alpha$. Then, for each $l=1,\dots,D$ and every $n\geq 1$
\begin{displaymath}
\int ||\phi\circ f||^{\alpha+\epsilon}p_n^{(l)}(\mathrm d f)<+\infty.
\end{displaymath}
\end{lem}
\begin{proof}
Since $\alpha<2$ and $\beta<1$, then there exists $\epsilon>0$ that satisfies the conditions in the statement. Moreover, since $p_n^{(l)}$ is an $\alpha$-stable distribution and since $(\alpha+\epsilon)\beta<\alpha$, then $\int |f\mathbf 1_r|^{(\alpha+\epsilon)\beta}p_n^{(l)}(\mathrm d f)<+\infty$. The thesis follows by noticing that, since $\alpha+\epsilon<2$, then there exist $c, d\in \mathbb R_+$ such that it holds true
\begin{align*}
||\phi\circ f||^{\alpha+\epsilon} & \leq\sum_{r=1}^{k}|\phi(f\mathbf 1_r)|^{\alpha+\epsilon}\leq \sum_{r=1}^{k}(a+b|f\mathbf 1_r|^{\beta})^{\alpha+\epsilon}\leq c+d	\sum_{r=1}^{k}|f\mathbf 1_r|^{(\alpha+\epsilon)\beta}.
\end{align*}
\end{proof}

Now, we present the induction hypothesis over the NN's layers, which is critical to prove Theorem \ref{teo1}. In particular, it is assumed that, for every index $l'<l$ and for $\epsilon$ as specified in Lemma \ref{lem1}, as $n\rightarrow+\infty$
\begin{equation}\label{induction}
\quad p_{n}^{(l')}\overset{\text{a.s.}}{\longrightarrow}q^{(l')} \qquad\text{in the weak topology},
\end{equation}
\begin{equation}\label{induction2}
\int ||\phi\circ f||^{\alpha+\epsilon} p_n^{(l')}(\mathrm d f)\overset{\text{ a.s.}}{\longrightarrow}\int ||\phi\circ f||^{\alpha+\epsilon} q^{(l')}(\mathrm d f)
\end{equation}
and
\begin{equation}\label{induction3}
\int |(\phi\circ f)\mathbf t|^\alpha p_n^{(l')}(\mathrm d f)\stackrel{\text{a.s.}}{\longrightarrow }\int |(\phi\circ f)\mathbf t|^\alpha q^{(l')}(\mathrm d f) \qquad\text{ for every }\mathbf t\in \mathbb R^k,
\end{equation}
with $q^{(l')}$ being the $\text{St}_{k}(\alpha,\Gamma^{(l')})$ and with $\Gamma^{(l')}$ being specified by the recurrence relation \eqref{lim_spect}, with $\Gamma^{(1)}$ being defined in \eqref{spectral1}. Note that the both the integrals that appear in Equation \eqref{induction2} and Equation \eqref{induction3} are finite by Lemma \ref{lem1} and since $|(\phi\circ f)\mathbf t|\leq ||\mathbf t|| \;||\phi\circ f||$. The induction hypothesis is trivially true for $l'=1$ since $p_n^{(1)}=q^{(1)}$. Note that, according to the induction hypothesis, as $n\rightarrow+\infty$, $p_{n}^{(l')}$ has a deterministic limit at $\text{St}_{k}(\alpha,\Gamma^{(l')})$, for every $l'<l$. The proof of Theorem \ref{teo1} is presented in three
steps. First, it is shown that \eqref{induction}, \eqref{induction2} and \eqref{induction3} hold true for $l'=l$. Then for each $i\geq1$ we show that
\begin{equation}\label{eq:th1}
{f}_{i}^{(l)}(\mathbf{X},n)\overset{w}{\longrightarrow}f_{i}^{(l)}(\mathbf{X})
\end{equation}
as $n\rightarrow+\infty$, where $f_{i}^{(l)}(\mathbf{X})$ is distributed as $\text{St}_{k}(\alpha,\Gamma^{(l)})$, with $\Gamma^{(l)}$ being defined in \eqref{lim_spect}.  Then, as $n\rightarrow+\infty$, the weak convergence of $(f^{(l)}_{i}(\mathbf{X},n))_{i\geq1}$ to $(f^{(l)}_{i}(\mathbf{X}))_{i\geq1}$ follows directly by combining \eqref{eq:th1} with some standard arguments that exploit the finite-dimensional projections of $(f^{(l)}_{i}(\mathbf{X},n))_{i\geq1}$ \citep{Bil(99)}.

\subsection{Induction step}
We start by proving the induction hypothesis given by the combination of Equation \eqref{induction}, Equation \eqref{induction2} and Equation \eqref{induction3}. In particular, let $\mathbf{t}$ be a $k$-dimensional (column) vector, then we can write the following:
\begin{align}\label{eq:espan1}
 & \mathbb{E}_{n,l-2}[\text{e}^{\textrm{i}{f}_i^{(l)}(\mathbf{X},n)\mathbf{t}}]                                                                                                                                                                                                                                                                           \\
 & \notag\quad=\mathbb{E}_{n,l-2}\left[\exp\left\{-\int_{\mathbb{S}^{k-1}}|\mathbf{s}\mathbf{t}|^{\alpha}\Gamma_n^{(l)}(\text{d}\mathbf{s})\right\}\right]                                                                                                                                                                                                \\
 & \notag\quad=\mathbb{E}_{n,l-2}\left[\exp\left\{-\int_{\mathbb{S}^{k-1}}|\mathbf{s}\mathbf{t}|^{\alpha}\left(||\sigma_{b}\mathbf{1}^{T}||^{\alpha}\zeta_{\frac{\mathbf{1}^{T}}{||\mathbf{1}^{T}||}}\right)(\text{d}\mathbf{s})\right\}\right]                                                                                                           \\
 & \notag\quad\quad\times\mathbb{E}_{n,l-2}\left[\exp\left\{-\int_{\mathbb{S}^{k-1}}|\mathbf{s}\mathbf{t}|^{\alpha}\left(\frac{1}{n}\sum_{j=1}^{n}||\sigma_{w}(\phi\circ f_{j}^{(l-1)}(\mathbf{X},n))||^{\alpha}\zeta_{\frac{\phi\circ f_{j}^{(l-1)}(\mathbf{X},n)}{||\phi\circ f_{j}^{(l-1)}(\mathbf{X},n)||}}\right)(\text{d}\mathbf{s})\right\}\right] \\
 & \notag\quad=\exp\left\{-\int_{\mathbb{S}^{k-1}}|\mathbf{s}\mathbf{t}|^{\alpha}
\left(||\sigma_{b}\mathbf{1}^{T}||^{\alpha}\zeta_{\frac{\mathbf{1}^{T}}{||\mathbf{1}^{T}||}}\right)(\text{d}\mathbf{s})\right\}                                                                                                                                                                                                                           \\
 & \notag\quad\quad\times
\left(\int\exp\left\{-\int_{\mathbb{S}^{k-1}}|
\mathbf{s}\mathbf{t}|^{\alpha}\left(\frac{1}{n}||
\sigma_{w}(\phi\circ{f})||^{\alpha}\zeta_{\frac{\phi\circ{f}}{||\phi\circ{f}||}}
\right)(\text{d}\mathbf{s})\right\}p_{n}^{(l-1)}(\text{d}{f})\right)^{n}.
\end{align}
Then \eqref{induction}, with $l'=l$, follows by combining the conditional characteristic function \eqref{eq:espan1} with the following lemma.

\begin{lem}\label{lem4}
If \eqref{induction}, \eqref{induction2} and \eqref{induction3} hold for every $l'\leq l-1$, then
\begin{displaymath}
\int ||\phi\circ{f}||^\alpha \left[1-\exp\left\{-\int_{\mathbb{S}^{k-1}}|\mathbf{s}\mathbf{t}|^{\alpha}\left(\frac{1}{n}||\sigma_{w}(\phi\circ{f})||^{\alpha}\zeta_{\frac{\phi\circ{f}}{||\phi\circ{f}||}}\right)(\mathrm{d}\mathbf{s})\right\}\right]p_{n}^{(l-1)}(\mathrm{d}{f})\overset{\text{a.s.}}{\longrightarrow} 0,
\end{displaymath}
as $n\rightarrow+\infty$.
\end{lem}
\begin{proof}
Let $\epsilon>0$ be as specified in \eqref{induction2}, i.e. as specified in Lemma \ref{lem1}, and also let $p=(\alpha+\epsilon)/\alpha$ and $q=(\alpha+\epsilon)/\epsilon$. Then, it holds that $1/p+1/q=1$. Accordingly, by means of H\"older inequality we can write that
\begin{align*}
 & \int ||\phi\circ{f}||^\alpha\left[1-\exp\left\{-\int_{\mathbb{S}^{k-1}}|\mathbf{s}\mathbf{t}|^{\alpha}\left(\frac{1}{n}||\sigma_{w}(\phi\circ{f})||^{\alpha}\zeta_{\frac{\phi\circ{f}}{||\phi\circ{f}||}}\right)(\text{d}\mathbf{s})\right\}\right]p_{n}^{(l-1)}(\text{d}{f})                   \\
 & \quad\leq \left(\int ||\phi\circ{f}||^{\alpha p}p_{n}^{(l-1)}(\text{d}{f})\right)^{1/p}                                                                                                                                                                                                         \\
 & \quad\quad\times\left(\int \left[1-\exp\left\{-\int_{\mathbb{S}^{k-1}}|\mathbf{s}\mathbf{t}|^{\alpha}\left(\frac{1}{n}||\sigma_{w}(\phi\circ{f})||^{\alpha}\zeta_{\frac{\phi\circ{f}}{||\phi\circ{f}||}}\right)(\text{d}\mathbf{s})\right\}\right]^{q}p_{n}^{(l-1)}(\text{d}{f}) \right)^{1/q}.
\end{align*}
Note that we defined $p=(\alpha+\epsilon)/\alpha$ and $q=(\alpha+\epsilon)/\epsilon$, i.e. we set $q>1$. Accordingly, we can write that
\begin{align*}
 & \left(\int ||\phi\circ{f}||^{\alpha p}p_{n}^{(l-1)}(\text{d}{f})\right)^{1/p}                                                                                                                                                                                                                  \\
 & \quad\quad\times\left(\int \left[1-\exp\left\{-\int_{\mathbb{S}^{k-1}}|\mathbf{s}\mathbf{t}|^{\alpha}\left(\frac{1}{n}||\sigma_{w}(\phi\circ{f})||^{\alpha}\zeta_{\frac{\phi\circ{f}}{||\phi\circ{f}||}}\right)(\text{d}\mathbf{s})\right\}\right]^{q}p_{n}^{(l-1)}(\text{d}{f}) \right)^{1/q} \\
 & \quad \leq  \left(\int ||\phi\circ{f}||^{\alpha +\epsilon}p_{n}^{(l-1)}(\text{d}{f})\right)^{1/p}                                                                                                                                                                                              \\
 & \quad\quad\times\left(\int \left[1-\exp\left\{-\int_{\mathbb{S}^{k-1}}|\mathbf{s}\mathbf{t}|^{\alpha}\left(\frac{1}{n}||\sigma_{w}(\phi\circ{f})||^{\alpha}\zeta_{\frac{\phi\circ{f}}{||\phi\circ{f}||}}\right)(\text{d}\mathbf{s})\right\}\right]p_{n}^{(l-1)}(\text{d}{f}) \right)^{1/q}     \\
 & \quad \leq  \left(\int ||\phi\circ{f}||^{\alpha +\epsilon}p_{n}^{(l-1)}(\text{d}{f})\right)^{1/p}                                                                                                                                                                                              \\
 & \quad\quad\times\left(\int \left[ \int_{\mathbb{S}^{k-1}}|\mathbf{s}\mathbf{t}|^{\alpha}\left(\frac{1}{n}||\sigma_{w}(\phi\circ{f})||^{\alpha}\zeta_{\frac{\phi\circ{f}}{||\phi\circ{f}||}}\right)(\text{d}\mathbf{s})\right] p_{n}^{(l-1)}(\text{d}{f}) \right)^{1/q}                         \\
 & \quad \leq  \left(\int ||\phi\circ{f}||^{\alpha +\epsilon}p_{n}^{(l-1)}(\text{d}{f})\right)^{1/p}
\left(\frac{||\mathbf t||^\alpha}{n} \int
||\sigma_{w}(\phi\circ{f})||^{\alpha} p_{n}^{(l-1)}(\text{d}{f}) \right)^{1/q}\stackrel{\text{a.s.}}\longrightarrow 0,
\end{align*}
as $n\rightarrow+\infty$, by \eqref{induction2}.
\end{proof}

We prove Equation \eqref{induction} by combining the conditional characteristic function \eqref{eq:espan1} with \eqref{induction}, \eqref{induction2} and \eqref{induction3} for $l'=l-1$, and then by Lemma \ref{lem1} and Lemma \ref{lem4}.
By combining \eqref{eq:espan1} with Lemma \ref{lem1}, we write
\begin{align*}
 & \mathbb{E}_{n,l-2}[\text{e}^{\textrm{i}{f}_i^{(l)}(\mathbf{X},n)\mathbf{t}}]                                                                                                                                                                                             \\
 & \notag\quad=\exp\left\{-\int_{\mathbb{S}^{k-1}}|\mathbf{s}\mathbf{t}|^{\alpha}\left(||\sigma_{b}\mathbf{1}^{T}||^{\alpha}\zeta_{\frac{\mathbf{1}^{T}}{||\mathbf{1}^{T}||}}\right)(\text{d}\mathbf{s})\right\}                                                            \\
 & \quad\quad\times\left(\int\exp\left\{-\int_{\mathbb{S}^{k-1}}|\mathbf{s}\mathbf{t}|^{\alpha}\left(\frac{1}{n}||\sigma_{w}(\phi\circ{f})||^{\alpha}\zeta_{\frac{\phi\circ{f}}{||\phi\circ{f}||}}\right)(\text{d}\mathbf{s})\right\}p_{n}^{(l-1)}(\text{d}{f})\right)^{n}.
\end{align*}
Now, by means of a direct application of Lagrange theorem, there exists a random variable  $\theta_{n}\in[0,1]$ such that
\begin{align*}
 & \exp\left\{-\int_{\mathbb{S}^{k-1}}|\mathbf{s}\mathbf{t}|^{\alpha}\left(\frac{1}{n}||\sigma_{w}(\phi\circ{f})||^{\alpha}\zeta_{\frac{\phi\circ{f}}{||\phi\circ{f}||}}\right)(\text{d}\mathbf{s})\right\}                                                          \\
 & \quad=1-\int_{\mathbb{S}^{k-1}}|\mathbf{s}\mathbf{t}|^{\alpha}\left(\frac{1}{n}||\sigma_{w}(\phi\circ{f})||^{\alpha}\zeta_{\frac{\phi\circ{f}}{||\phi\circ{f}||}}\right)(\text{d}\mathbf{s})                                                                      \\
 & \quad\quad\quad\quad\times\exp\left\{-\theta_{n}\int_{\mathbb{S}^{k-1}}|\mathbf{s}\mathbf{t}|^{\alpha}\left(\frac{1}{n}||\sigma_{w}(\phi\circ{f})||^{\alpha}\zeta_{\frac{\phi\circ{f}}{||\phi\circ{f}||}}\right)(\text{d}\mathbf{s})\right\}                      \\
 & \quad=1-\int_{\mathbb{S}^{k-1}}|\mathbf{s}\mathbf{t}|^{\alpha}\left(\frac{1}{n}||\sigma_{w}(\phi\circ{f})||^{\alpha}\zeta_{\frac{\phi\circ{f}}{||\phi\circ{f}||}}\right)(\text{d}\mathbf{s})                                                                      \\
 & \quad\quad\quad\quad+\int_{\mathbb{S}^{k-1}}|\mathbf{s}\mathbf{t}|^{\alpha}\left(\frac{1}{n}||\sigma_{w}(\phi\circ{f})||^{\alpha}\zeta_{\frac{\phi\circ{f}}{||\phi\circ{f}||}}\right)(\text{d}\mathbf{s})                                                         \\
 & \quad\quad\quad\quad\quad\times\left(1-\exp\left\{-\theta_{n}\int_{\mathbb{S}^{k-1}}|\mathbf{s}\mathbf{t}|^{\alpha}\left(\frac{1}{n}||\sigma_{w}(\phi\circ{f})||^{\alpha}\zeta_{\frac{\phi\circ{f}}{||\phi\circ{f}||}}\right)(\text{d}\mathbf{s})\right\}\right).
\end{align*}
Now, since by Lemma \ref{lem4},
\begin{align*}
0 & \leq\int\int_{\mathbb{S}^{k-1}}
|\mathbf{s}\mathbf{t}|^{\alpha}\left(||\sigma_{w}(\phi\circ{f})||^{\alpha}\zeta_{\frac{\phi\circ{f}}{||\phi\circ{f}||}}\right)(\text{d}\mathbf{s})                                                                                                                                      \\
  & \quad\left.\times\left[1-\exp\left\{-\theta_{n}\int_{\mathbb{S}^{k-1}}|\mathbf{s}\mathbf{t}|^{\alpha}\left(\frac{1}{n}||\sigma_{w}(\phi\circ{f})||^{\alpha}\zeta_{\frac{\phi\circ{f}}{||\phi\circ{f}||}}\right)(\text{d}\mathbf{s})\right\}\right]p_{n}^{(l-1)}(\text{d}{f})\right. \\
  & \leq ||\mathbf t||^\alpha \sigma_w^\alpha
\int ||\phi\circ f||^\alpha \left[1-\exp\left\{-\int_{\mathbb{S}^{k-1}}|\mathbf{s}\mathbf{t}|^{\alpha}\left(\frac{1}{n}||\sigma_{w}(\phi\circ{f})||^{\alpha}\zeta_{\frac{\phi\circ{f}}{||\phi\circ{f}||}}\right)(\text{d}\mathbf{s})\right\}\right]p_{n}^{(l-1)}(\text{d}{f})
\stackrel{\text{a.s.}}{\longrightarrow}0,
\end{align*}
as $n\rightarrow\infty$, then
\begin{align*}
 & \mathbb{E}_{n,l-2}[\text{e}^{\textrm{i}{f}_i^{(l)}(\mathbf{X},n)\mathbf{t}}]                                                                                                                            \\
 & \quad=\exp\left\{-\int_{\mathbb{S}^{k-1}}|\mathbf{s}\mathbf{t}|^{\alpha}\left(||\sigma_{b}\mathbf{1}^{T}||^{\alpha}\zeta_{\frac{\mathbf{1}^{T}}{||\mathbf{1}^{T}||}}\right)(\text{d}\mathbf{s})\right\}
\Biggl(1-\frac{1}{n}
\int
\sigma_w^\alpha|(\phi\circ f)\mathbf{t}|^{\alpha}p_{n}^{(l-1)}(\text{d}{f})
+o\biggl(\frac{1}{n}\biggr)
\Biggr)^{n}                                                                                                                                                                                                \\
 & \quad\stackrel{a.s.}{\longrightarrow}
\exp\left\{-\int_{\mathbb{S}^{k-1}}|\mathbf{s}\mathbf{t}|^{\alpha}\left(||\sigma_{b}\mathbf{1}^{T}||^{\alpha}\zeta_{\frac{\mathbf{1}^{T}}{||\mathbf{1}^{T}||}}\right)(\text{d}\mathbf{s})\right\}\exp\Biggl\{-\int
\sigma_w^\alpha|(\phi\circ f)\mathbf{t}|^{\alpha}p_{n}^{(l-1)}(\text{d}{f})
\Biggr\},
\end{align*}
as $n\rightarrow\infty$, by \eqref{induction3}, Lemma \ref{lem4}, and since $\text{e}^{-x}=\lim_{n\rightarrow+\infty}(1-x_n/n)^{n}$, as $x_n\rightarrow x\in (0,+\infty)$. By \cite[Theorem 2]{blackwell},
\begin{align}
\label{eq:ind1}
 & \mathbb{E}_{n,l-1}[\text{e}^{\textrm{i}{f}_i^{(l)}(\mathbf{X},n)\mathbf{t}}]=\mathbb{E}_{n,l-1}[\mathbb{E}_{n,l-2}[\text{e}^{\textrm{i}{f}_i^{(l)}(\mathbf{X},n)\mathbf{t}}]]\nonumber                                                                                              \\
 & \quad\stackrel{\text{a.s.}}{\longrightarrow}\mathbb E_{\infty,l-1}\Biggl[\exp\left\{-\int_{\mathbb{S}^{k-1}}|\mathbf{s}\mathbf{t}|^{\alpha}\left(||\sigma_{b}\mathbf{1}^{T}||^{\alpha}\zeta_{\frac{\mathbf{1}^{T}}{||\mathbf{1}^{T}||}}\right)(\text{d}\mathbf{s})\right\}\nonumber \\
 & \quad\quad\quad\quad\times\exp\left\{-\int\int_{\mathbb{S}^{k-1}}|\mathbf{s}\mathbf{t}|^{\alpha}\left(||\sigma_{w}(\phi\circ{f})||^{\alpha}\zeta_{\frac{\phi\circ{f}}{||\phi\circ{f}||}}\right)(\text{d}\mathbf{s})q^{(l-1)}(\text{d}{f})\right\}\Biggr]\nonumber                   \\
 & \quad\quad\quad=\exp\left\{-\int_{\mathbb{S}^{k-1}}|\mathbf{s}\mathbf{t}|^{\alpha}\left(||\sigma_{b}\mathbf{1}^{T}||^{\alpha}\zeta_{\frac{\mathbf{1}^{T}}{||\mathbf{1}^{T}||}}\right)(\text{d}\mathbf{s})\right\}\nonumber                                                          \\
 & \quad\quad\quad\quad\times\exp\left\{-\int\int_{\mathbb{S}^{k-1}}|\mathbf{s}\mathbf{t}|^{\alpha}\left(||\sigma_{w}(\phi\circ{f})||^{\alpha}\zeta_{\frac{\phi\circ{f}}{||\phi\circ{f}||}}\right)(\text{d}\mathbf{s})q^{(l-1)}(\text{d}{f})\right\}
\end{align}
as $n\rightarrow+\infty$, where the last equality holds true since $q^{(l-1)}$ is deterministic. Therefore, $n\rightarrow\infty$, $p_n^{(l)}$ converges a.s. in the weak topology to $q^{(l)}=\text{St}_{k}(\alpha,\Gamma^{(l)})$, thus proving the induction step for \eqref{induction}, where we set
\begin{equation*}
\Gamma^{(l)}=||\sigma_{b}\mathbf{1}^{T}||^{\alpha}\zeta_{\frac{\mathbf{1}^{T}}{||\mathbf{1}^{T}||}}+\int||\sigma_{w}(\phi\circ{f})||^{\alpha}\zeta_{\frac{\phi\circ{f}}{||\phi\circ{f}||}}q^{(l-1)}(\text{d}{f}).
\end{equation*}

Now, we prove that Equation \eqref{induction2} holds true for $l'=l$. The proof is based on a uniform integrability argument. Since $p_n^{(l)}$ converges a.s. to $q^{(l)}$, with respect to the weak topology, then, for every $r=1,\dots,k$,
\begin{displaymath}
\left(\int_{\mathbb S^{k-1}}|\mathbf s\mathbf 1_r|^\alpha \Gamma_n^{(l)}(\mathrm d \mathbf s)\right)^{1/\alpha}
=\sigma_n^{(l)}(r)
\stackrel{\text{a.s.}}{\longrightarrow } \sigma^{(l)}(r):=
\left(\int_{\mathbb S^{k-1}}|\mathbf s\mathbf 1_r|^\alpha
\Gamma^{(l)}(\mathrm d \mathbf s)\right)^{1/\alpha}.
\end{displaymath}
Denoting by $S_{\alpha,1}$ a random variable distributed according to $St(\alpha,1)$, we can write, for every $r=1,\dots,k$,
\begin{align*}
\int |f \mathbf 1_r|^{(\alpha+\epsilon)\beta}p_n^{(l)}(\mathrm d f) &
=\mathbb E[S_{\alpha,1}^{(\alpha+\epsilon)\beta}]
(\sigma_n^{(l)}(r))^{(\alpha+\epsilon)\beta}\stackrel{\text{a.s.}}{\longrightarrow}
\mathbb E[S_{\alpha,1}^{(\alpha+\epsilon)\beta}]
(\sigma^{(l)}(r))^{(\alpha+\epsilon)\beta}<+\infty.
\end{align*}
It follows that, for every $r=1,\dots, k$, the functional $|f \mathbf 1_r|^{(\alpha+\epsilon)\beta}$ is a.s. uniformly integrable with respect to $p_n^{(l)}$, that is
$$
\sup_n\int_{\{|f \mathbf 1_r|^{(\alpha+\epsilon)\beta}>a\}} |f \mathbf 1_r|^{(\alpha+\epsilon)\beta}p_n^{(l)}(df)\stackrel{a.s.}{\longrightarrow}0 \quad \mbox{as }a\rightarrow\infty.
$$
To prove it, fix $\omega\in\Omega$ such that $p_n^{(l)}(\omega)$ converges weakly to $q^{(l)}$, as $n\rightarrow+\infty$, and let $(\tilde f_n^{(\omega)})_{n\geq 1}$ and $\tilde f^{(\omega)}$ be random vectors defined on a probability space $(\tilde \Omega,\tilde{\mathcal F},\tilde{\mathbb P})$ with distribution $ (p_n^{(l)}(\omega))_{n\geq 1}$ and $q^{(l)}$, respectively. Then, under $\tilde{\mathbb P}$, the sequence $(|\tilde f_n^{(\omega)} \mathbf 1_r|^{(\alpha+\epsilon)\beta})_{n\geq 1}$ converges in distribution, as $n\rightarrow+\infty$, to $|\tilde f^{(\omega)} \mathbf 1_r|^{(\alpha+\epsilon)\beta}$ and also
\begin{align*}
\tilde{\mathbb E}(|\tilde f_n^{(\omega)} \mathbf 1_r|^{(\alpha+\epsilon)\beta}) & =	\int |f \mathbf 1_r|^{(\alpha+\epsilon)\beta}p_n^{(l)}(\omega)(\mathrm d f) \\
                                                                                & \quad \rightarrow
\int |f \mathbf 1_r|^{(\alpha+\epsilon)\beta}q^{(l)}(\mathrm d f)=\tilde{\mathbb E}(|\tilde f^{(\omega)} \mathbf 1_r|^{(\alpha+\epsilon)\beta})<+\infty.
\end{align*}
By uniform integrability,
\begin{align*}
\sup_n & \int_{\{|f \mathbf 1_r|^{(\alpha+\epsilon)\beta}>a\}} |f \mathbf 1_r|^{(\alpha+\epsilon)\beta}p_n^{(l)}(\omega)(df)=\sup_n \tilde{\mathbb E}\left[|\tilde f_n^{(\omega)} \mathbf 1_r|^{(\alpha+\epsilon)\beta} I_{(a,\infty)}(|\tilde f_n^{(\omega)} \mathbf 1_r|^{(\alpha+\epsilon)\beta})\right]\rightarrow 0 ,
\end{align*}
as $a\rightarrow\infty$.
Thus, $|f \mathbf 1_r|^{(\alpha+\epsilon)\beta}$ is a.s. uniformly integrable with respect to $p_n^{(l)}$, for $r=1,\dots, k$. Since for some $c,d\in\mathbb R_+$
\begin{align*}
||\phi\circ f||^{\alpha+\epsilon} & \leq \sum_{r=1}^k |\phi(f\mathbf 1_r)|^{\alpha+\epsilon}\leq \sum_{r=1}^k(a+b|f\mathbf 1_r|^\beta)^{\alpha+\epsilon}\leq c+d\sum_{r=1}^k |f\mathbf 1_r|^{(\alpha+\epsilon)\beta},
\end{align*}
then $||\phi\circ f||^{\alpha+\epsilon}$ is a.s. uniformly integrable with respect to $p_n^{(l)}$. Since $p_n^{(l)}$ converges a.s.to $q^{(l)}$, then $\int ||\phi\circ f||^{\alpha+\epsilon}p_n^{(l)}(\mathrm d f)$ converges a.s. to $\int ||\phi\circ f||^{\alpha+\epsilon}q^{(l)}(\mathrm d f)$. This proves the induction step for \eqref{induction2}. Finally, we prove that Equation \eqref{induction3} holds true for $l'=l$. In particular, we observe that $|(\phi\circ f)\mathbf t|\leq ||\mathbf t||\;||\phi\circ f||$. Then, $|(\phi\circ f)\mathbf t|^{\alpha}$ is also a.s. uniformly integrable with respect to $p_n^{(l)}$. Equation \eqref{induction3} with $l'=l$ follows from this and \eqref{induction} with $l'=l$. This completes the proof of the induction hypothesis.

\subsection{Weak convergence of $f^{(l)}_{i}(\mathbf{X},n)$}

We prove Equation \eqref{eq:th1}. In particular, by means of \eqref{eq:ind1} and dominated convergence theorem, which is applied to the sequence $(\mathbb E_{n,l-2}[e^{if_i^{(l)}(X,n)\mathbf t}])_{n\geq 1}$ of uniformly bounded random variables, we can write
\begin{align*}
 & \mathbb{E}[\text{e}^{\textrm{i}{f}_i^{(l)}(\mathbf{X},n)\mathbf{t}}]                                                                                                                                                                          \\
 & \quad\rightarrow\mathbb E
\Biggl[\exp\left\{-\int_{\mathbb{S}^{k-1}}|\mathbf{s}\mathbf{t}|^{\alpha}\left(||\sigma_{b}\mathbf{1}^{T}||^{\alpha}\zeta_{\frac{\mathbf{1}^{T}}{||\mathbf{1}^{T}||}}\right)(\text{d}\mathbf{s})\right\}\exp
\left\{-\int \sigma_w^\alpha |(\phi\circ f)\mathbf{t}|^{\alpha}
q^{(l-1)}(\text{d}{f})
\right\}\Biggr]                                                                                                                                                                                                                                  \\
 & \quad\quad =\exp\left\{-\int_{\mathbb{S}^{k-1}}|\mathbf{s}\mathbf{t}|^{\alpha}\left(||\sigma_{b}\mathbf{1}^{T}||^{\alpha}\zeta_{\frac{\mathbf{1}^{T}}{||\mathbf{1}^{T}||}}\right)(\text{d}\mathbf{s})\right\}                                 \\
 & \quad\quad\quad\times\exp\left\{-\int\int_{\mathbb{S}^{k-1}}|\mathbf{s}\mathbf{t}|^{\alpha}\left(||\sigma_{w}(\phi\circ{f})||^{\alpha}\zeta_{\frac{\phi\circ{f}}{||\phi\circ{f}||}}\right)(\text{d}\mathbf{s})q^{(l-1)}(\text{d}{f})\right\}.
\end{align*}
That is, $f_{i}^{(l)}(\mathbf{X},n)$ converges weakly, as $n\rightarrow+\infty$, to $f_{i}^{(l)}(\mathbf{X})$ distributed as $\text{St}_{k}(\alpha,\Gamma^{(l)})$, for each $i\geq1$, where
\begin{equation*}
\Gamma^{(l)}=||\sigma_{b}\mathbf{1}^{T}||^{\alpha}\zeta_{\frac{\mathbf{1}^{T}}{||\mathbf{1}^{T}||}}+\int||\sigma_{w}(\phi\circ{f})||^{\alpha}\zeta_{\frac{\phi\circ{f}}{||\phi\circ{f}||}}q^{(l-1)}(\text{d}{f}),
\end{equation*}
with $q^{(l-1)}$ being the distribution of $f_{i}^{(l-1)}(\mathbf{X})$, for $l=2,\ldots,D$. This result completes the proof of \eqref{eq:th1}.

\subsection{Weak convergence of $(f^{(l)}_{i}(\mathbf{X},n))_{i\geq1}$}

By Cram\'er-Wold theorem \citep{Bil(99)} the convergence of $(f^{(l)}_{i}(\mathbf{X},n))_{i\geq1}$ to some limit is equivalent to convergence on all possible linear projections of $(f^{(l)}_{i}(\mathbf{X},n))_{i\geq1}$ to the corresponding real-valued random variable. Let $\mathcal{L}\subset\mathbb{N}$ and let $\{p_{i}\}_{i\in\mathcal{L}}$ such that $p_{i}\in(0,1)$ and $\sum_{i\in\mathcal{L}}p_{i}=1$. Then, we consider the linear projection
\begin{align*}
T^{(l)}(\mathcal{L},p,\mathbf{X},n) & =\sum_{i\in \mathcal{L}}p_{i}[{f}_{i}^{(l)}(\mathbf{X},n)-b_{i}^{(l)}\mathbf{1}^{T}]                                                                                                                     \\
                                    & =\sum_{i\in \mathcal{L}}p_{i}\left[\frac{1}{n^{1/\alpha}}\sum_{j=1}^{n}w_{i,j}^{(l)}(\phi\circ {f}_{j}^{(l-1)}(\mathbf{X},n))\right]                                                                     \\
                                    & =\frac{1}{n^{1/\alpha}}\sum_{j=1}^{n}\sum_{i\in \mathcal{L}}p_{i}w_{i,j}^{(l)}(\phi\circ{f}_{j}^{(l-1)}(\mathbf{X},n))=\frac{1}{n^{1/\alpha}}\sum_{j=1}^{n}\gamma_{j}^{(l)}(\mathcal{L},p,\mathbf{X},n),
\end{align*}
where we set $\gamma_{j}^{(l)}(\mathcal{L},p,\mathbf{X},n)=\sum_{i\in \mathcal{L}}p_{i}w_{i,j}^{(l)}(\phi\circ{f}_{j}^{(l-1)}(\mathbf{X},n))$ for $j=1,\ldots,n$. Then, we can write that
\begin{align*}
 & \varphi_{T^{(l)}(\mathcal{L},p,\mathbf{X},n)\,|\,\{{f}_{j}^{(l-1)}(\mathbf{X},n)\}_{j\geq 1}}(\mathbf{t})                                                                                                                                                                                                                           \\
 & \quad=\mathbb{E}[e^{\textrm{i}T^{(l)}(\mathcal{L},p,\mathbf{X},n)\mathbf{t}}\,|\,\{{f}_{j}^{(l-1)}(\mathbf{X},n)\}_{j\geq 1}]                                                                                                                                                                                                       \\
 & \quad=\mathbb{E}\left[\exp\left\{\left[\frac{1}{n^{1/\alpha}}\sum_{j=1}^{n}\sum_{i\in \mathcal{L}}p_{i}w_{i,j}^{(l)}(\phi\circ{f}_{j}^{(l-1)}(\mathbf{X},n))\right](\textrm{i}\mathbf{t})\right\}\,|\,\{{f}_{j}^{(l-1)}(\mathbf{X},n)\}_{j\geq 1}\right]                                                                            \\
 & \quad=\prod_{j=1}^{n}\prod_{i\in\mathcal{L}}\mathbb{E}\left[\exp\left\{\frac{1}{n^{1/\alpha}}p_{i}w_{i,j}^{(l)}(\phi\circ{f}_{j}^{(l-1)}(\mathbf{X},n))(\textrm{i}\mathbf{t})\,|\,\{{f}_{j}^{(l-1)}(\mathbf{X},n)\}_{j\geq 1}\right\}\right]                                                                                        \\
 & \quad=\prod_{j=1}^{n}\prod_{i\in\mathcal{L}}\text{e}^{-\frac{p^{\alpha}_{i}\sigma_{w}^{\alpha}}{n}|(\phi\circ f_{j}^{(l-1)}(\mathbf{X},n))\mathbf{t}|^{\alpha}}                                                                                                                                                                     \\
 & \quad=\exp\left\{-\int_{\mathbb{S}^{k-1}}|\mathbf{s}\mathbf{t}|^{\alpha}\left(\frac{1}{n}\sum_{j=1}^{n}\sum_{i\in\mathcal{L}}||p_{i}\sigma_{w}(\phi\circ f_{j}^{(l-1)}(\mathbf{X},n))||^{\alpha}\zeta_{\frac{\phi\circ f_{j}^{(l-1)}(\mathbf{X},n)}{||\phi\circ f_{j}^{(l-1)}(\mathbf{X},n)||}}\right)(\text{d}\mathbf{s})\right\}.
\end{align*}
That is,
\begin{align*}
 & T^{(l)}(\mathcal{L},p,\mathbf{X},n)\,|\,\{{f}_{j}^{(l-1)}(\mathbf{X},n)\}_{j\geq 1}\overset{\text{d}}{=}\mathbf{S}_{\alpha,\Gamma^{(l)}_{n,\mathcal L}},
\end{align*}
where $\mathbf{S}_{\alpha,\Gamma^{(l)}_n}$ is a random vector with symmetric $\alpha$-stable distribution and spectral measure of the form
\begin{equation*}
\Gamma^{(l)}_{n,\mathcal L}=\frac{1}{n}\sum_{j=1}^{n}\sum_{i\in\mathcal{L}}||p_{i}\sigma_{w}(\phi\circ f_{j}^{(l-1)}(\mathbf{X},n))||^{\alpha}\zeta_{\frac{\phi\circ f_{j}^{(l-1)}(\mathbf{X},n)}{||\phi\circ f_{j}^{(l-1)}(\mathbf{X},n)||}}.
\end{equation*}
Along lines similar to the proof of the large $n$ asymptotics for the $i$-th coordinate $f_{i}(\mathbf{X},n)$, we can show that
\begin{align*}
 & \mathbb{E}[\text{e}^{T^{(l)}(\mathcal{L},p,\mathbf{X},n)(\textrm{i}\mathbf{t})}]\rightarrow \exp\left\{-\int\int_{\mathbb{S}^{k-1}}|\mathbf{s}\mathbf{t}|^{\alpha}\left(\sum_{i\in\mathcal{L}}||p_{i}\sigma_{w}(\phi\circ{f})||^{\alpha}\zeta_{\frac{\phi\circ{f}}{||\phi\circ{f}||}}\right)(\text{d}\mathbf{s})q^{(l-1)}(\text{d}{f})\right\}
\end{align*}
as $n\rightarrow+\infty$. That is, the linear projection $T^{(l)}(\mathcal{L},p,\mathbf{X},n)$ converges weakly, as $n\rightarrow+\infty$, to $T^{(l)}(\mathcal{L},p,\mathbf{X})=\sum_{i\in \mathcal{L}}p_{i}[{f}_{i}^{(l)}(\mathbf{X})-b_{i}^{(l)}\mathbf{1}^{T}]$ where the ${f}_{i}^{(l)}(\mathbf{X})$ are i.i.d. according to $\text{St}_{k}(\alpha,\Gamma^{(l)})$, where we set
\begin{displaymath}
\Gamma^{(l)}=||\sigma_{b}\mathbf{1}^{T}||^{\alpha}\zeta_{\frac{\mathbf{1}^{T}}{||\mathbf{1}^{T}||}}+\int||\sigma_{w}(\phi\circ{f})||^{\alpha}\zeta_{\frac{\phi\circ{f}}{||\phi\circ{f}||}}q^{(l-1)}(\text{d}{f}),
\end{displaymath}
with $q^{(l-1)}$ being the distribution of $f_{i}^{(l-1)}(\mathbf{X})$, for any $l=2,\ldots,D$. Therefore, by means of Cram\'er-Wold theorem, $(f^{(l)}_{i}(\mathbf{X},n))_{i\geq1}$ converges weakly, as $n\rightarrow+\infty$, to the Stable SP $(f^{(l)}_{i}(\mathbf{X},n))_{i\geq1}$, as a process indexed by $\mathbf{X}$, whose distribution is $\otimes_{i\geq1}\text{St}_{k}(\alpha,\Gamma^{(l)})$. This completes the proof of Theorem \ref{teo1}

As a complement to the proof of Theorem \ref{teo1}, we show the consistency or compatibility of the finite-dimensional distributions of the Stable SP $(f_{i}^{(l)}(\mathbf{X}))_{i\geq1}$. In particular, proceeding by induction, we write
\begin{displaymath}
\varphi_{f_i^{(1)}(\mathbf{X})}(\bm t)=\exp\left\{
-\sigma_b^\alpha |\mathbf 1^T\mathbf t|^\alpha-\sum_{j=1}^I\sigma_w^\alpha |\mathbf x_j \mathbf  t|^\alpha
\right\}
\end{displaymath}
and, for $l>1$,
\begin{displaymath}
\varphi_{f_i^{(l)}(\mathbf{X})}(\bm t)=\exp\left\{
-\sigma_b^\alpha |\mathbf 1^T\mathbf t|^\alpha-\int \sigma_w^\alpha |(\phi\circ f)\mathbf  t|^\alpha q^{(l-1)}(df)
\right\}.
\end{displaymath}
Now, we define $\mathbf{X}_{\hat r}=[\bm{x}_1,\ldots,\bm{x}_{r-1},\bm{x}_{r+1},\ldots,\bm{x}_k]$, $\bm t_{\hat r}=[t_1,\dots,t_{r-1},t_{r+1},\dots, t_k]^T$, and
$\bm x_{j\hat r}=[x_{j1},\dots,x_{j,r-1},x_{j,r+1},\dots,x_{jk}]$. Moreover, for every $l=1,\dots,D$, we define a measure $q^{(l)}_{\hat r}$ as follows
\begin{displaymath}
q^{(l)}_{\hat r}(df_1,\dots df_{r-1},df_{r+1},\dots ,df_k)=
\int_{f_r\in\mathbb R} q^{(l)}(df_1,\dots,df_k).
\end{displaymath}
Then
\begin{align*}
 & \varphi_{f_i^{(1)}(\mathbf{X})}(t_1,\dots,t_{r-1},0,t_{r+1},\dots,t_k) \\
 & \quad=\exp\left\{
-\sigma_b^\alpha |\mathbf 1^T\mathbf t_{\hat r}|^\alpha-\sum_{j=1}^I\sigma_w^\alpha |\mathbf x_{j\hat r} \bm t_{\hat r}|^\alpha
\right\}=\varphi_{f_i^{(1)}(\mathbf{X}_{\hat r})}(\bm t_{\hat r}).
\end{align*}
Therefore, the consistency of the finite-dimensional distributions holds for $l=1$. Now suppose that the consistency holds true for every  $l'<l$. In particular,
$f_i^{(l-1)}(\mathbf{X}_{\hat r})$ has distribution $q_{\hat r}^{(l-1)}$. Then, we can write that
\begin{align*}
 & \varphi_{f_i^{(l)}(\mathbf{X})}(t_1,\dots,t_{r-1},0,t_{r+1},\dots,t_k) \\
 & \quad=\exp\left\{
-\sigma_b^\alpha |\mathbf 1^T\mathbf t_{\hat r}|^\alpha-\int \sigma_w^\alpha |(\phi\circ f)\mathbf  t_{\hat r}|^\alpha q^{(l-1)}_{r}(df)
\right\}=\varphi_{f_i^{(l)}(\mathbf{X}_{\hat r})}(\bm t_{\hat r}),
\end{align*}
which proves consistency or compatibility of the finite dimensional distributions for $l'=l$ of the Stable SP $(f_{i}^{(l)}(\mathbf{X}))_{i\geq1}$.


\section{Sup-norm convergence rates}\label{sec3}

In this section, we refine Theorem \ref{teo1} by establishing sup-norm convergence rates of the deep Stable NN $(f^{(l)}_{i}(\mathbf{X},n))_{i\geq1}$ to the Stable SP $(f^{(l)}_{i}(\mathbf{X}))_{i\geq1}$, under both the setting of  ``sequential growth" and ``joint growth" of the width over the NN's layers. Throughout this section we make the following assumptions:
\begin{equation}\label{cont}
\phi\mbox{ is continuous, strictly monotone and bounded}
\end{equation}
and
\begin{equation}\label{span}
\{\mathbf 1,\mathbf x_1,\dots,\mathbf x_I\} \mbox{ spans }\mathbb R^k.
\end{equation}

\subsection{The ``joint growth" setting}
The ``joint growth" setting consists in assuming that, for any $l=1,\ldots,D$, the width $n\rightarrow+\infty$ simultaneously over the first $l\geq1$ layers. We recall from Section \ref{sec2} that $f_{i}^{(1)}(\mathbf{X})\sim\text{St}_{k}(\alpha,\Gamma^{(1)})$ and $f_{i}^{(l)}(\mathbf{X},n)\,|\,\mathcal{G}_{n,l-1}\sim\text{St}(\alpha,\Gamma_{n}^{(l)})$
for any $l=2,\ldots,D$, where $\Gamma^{(1)}$ and $\Gamma_{n}^{(l)}$ are defined in \eqref{spectral1} and \eqref{spectral2}, respectively. In particular, $\Gamma^{(1)}$ and $\Gamma^{(l)}$ are finite random measures with (random) total masses given by
\begin{displaymath}
\Gamma^{(1)}(\mathbb{S}^{k-1})=\sigma_{b}^{\alpha}k^{\alpha/2}+\sigma_{w}^{\alpha}\sum_{j=1}^{I}||\mathbf{x}_{j}||^{\alpha}
\end{displaymath}
and
\begin{displaymath}
\Gamma_{n}^{(l)}(\mathbb{S}^{k-1})=\sigma_{b}^{\alpha}k^{\alpha/2}+\frac{\sigma_{w}^{\alpha}}{n}\sum_{j=1}^{n}||\phi\circ f_{j}^{(l-1)}(\mathbf{X},n)||^{\alpha},
\end{displaymath}
for $l=2,\ldots,D$, respectively. We recall from Theorem \ref{teo1} that $f^{(l)}_{i}(\mathbf{X})\sim\text{St}_{k}(\alpha,\Gamma^{(l)})$ for $l=1,\ldots,D$, where $\Gamma^{(l)}$ is a finite measure displayed in \eqref{lim_spect}. Under the assumption \eqref{cont}, it holds that
\begin{displaymath}
\overline \phi:=\sup_s |\phi(s)|<+\infty
\end{displaymath}
and
\begin{equation}\label{eq:gammabound}
\max(\Gamma_{n}^{(l)}(\mathbb{S}^{k-1}),\Gamma^{(l)}(\mathbb{S}^{k-1})
\leq \bar{\bm\gamma}:= \sigma_{b}^{\alpha}k^{\alpha/2}+\sigma_{w}^{\alpha}\bar{\phi}^{\alpha}k^{\alpha/2}
\end{equation}
for $l=2,\ldots,D$. Such a condition, together with the following lemma, allows to give an explicit uniform bound for the tails of the Stable distributions that are involved in the definition of the deep Stable NN.

\begin{lem}[\cite{Byc(93)}]\label{lemma:byc}
Let $\mathbf{S}$ be a $k$-dimensional random vector distributed as a symmetric centered $\alpha$-stable distribution with spectral measure $\Gamma$.  If
\begin{displaymath}
c(\alpha)=\left\{
\begin{array}{ll}
\frac{2^{\alpha/2-1}\pi^{1/2} (1+\tan^2(\pi\alpha/2))^{1/4}\cos(\pi\alpha/4)}{\alpha/2 \int_{0}^{+\infty} u^{-1-\alpha/2}\sin^2 u du} & \alpha\neq 1 \\
2^{3/2}\pi^{-1/2}                                                                                                                     & \alpha=1,
\end{array}
\right.
\end{displaymath}
then $\mathbb{P}[|\mathbf{S}|>R]\leq\epsilon$ for every $\epsilon>0$ whenever $R$ is such that $R^{\alpha/2}\geq c(\alpha)k\Gamma(\mathbb{S}^{k-1})^{1/\alpha}/\epsilon$.
\end{lem}

To establish sup-norm convergence rates, it is useful to consider linear transformations of the random vectors $f^{(l)}_{i}(\mathbf{X},n)$ and $f^{(l)}_{i}(\mathbf{X})$ \citep[Chapter 2]{Sam(94)}.
In particular, if $\mathbf{u}$ is a $k$-dimensional (column) vector then
\begin{itemize}
\item[i)] $f^{(1)}_{i}(\mathbf{X})\mathbf u \sim\text{St}(\alpha,(\bm\gamma^{(1)}(\mathbf{u}))^{1/\alpha})$, where
\begin{displaymath}
\bm\gamma^{(1)}(\mathbf{u})=\sigma_{b}^{\alpha}|\mathbf{1}^{T}\mathbf{u}|^{\alpha}+\sigma^{\alpha}_{w}\sum_{j=1}^{I}|\mathbf{x}_{j}\mathbf{u}|^{\alpha};
\end{displaymath}
\item[ii)] $f^{(l)}_{i}(\mathbf{X},n)\mathbf{u}\,|\,\mathcal{G}_{n,l-1}\sim\text{St}(\alpha,(\bm\gamma^{(l)}_{n}(\mathbf{u}))^{1/\alpha})$, where
\begin{displaymath}
\bm\gamma^{(l)}_{n}(\mathbf{u})=\sigma_{b}^{\alpha}|\mathbf{1}^{T}\mathbf{u}|^{\alpha}+\frac{\sigma^{\alpha}_{w}}{n}\sum_{j=1}^{n}|(\phi\circ f_{j}^{(l-1)}(\mathbf{X},n))\mathbf{u}|^{\alpha};
\end{displaymath}
\item[iii)] $f^{(l)}_{i}(\mathbf{X})\mathbf{u}\sim\text{St}(\alpha,(\bm\gamma^{(1)}(\mathbf{u}))^{1/\alpha})$, where
\begin{displaymath}
\bm\gamma^{(l)}(\mathbf{u})=\sigma^{\alpha}_{b}|\mathbf{1}^{T}\mathbf{u}|^{\alpha}+\sigma^{\alpha}_{w}\mathbb{E}[|(\phi\circ f_{j}^{(l-1)}(\mathbf{X}))\mathbf{u}|^{\alpha}].
\end{displaymath}
\end{itemize}
We denote by $\lambda_{k-1}$ the Lebesgue measure on $\mathbb{S}^{k-1}$. The next lemmas are critical to establish sup-norm convergence rates, as they show that the distributions $f^{(l)}_{i}(\mathbf{X},n)$ and $f^{(l)}_{i}(\mathbf{X})$ are absolutely continuous with respect to the Lebesgue measure. The next two lemmas deal with the distribution of $f^{(l)}_{i}(\mathbf{X})$.

\begin{lem}[\cite{Nol(10)}] \label{lemma:nolan}
Let $\mathbf{S}_{1}$ and $\mathbf{S}_{2}$ be $k$-dimensional random vectors distributed as symmetric $\alpha$-stable distributions with spectral measures $\Gamma_{1}$ and $\Gamma_{2}$, satisfying
\begin{displaymath}
\underline{\bm\gamma}:=\min\left(\inf_{\mathbf s\in \mathbb S^{k-1}}\bm\gamma_1(\mathbf s),\inf_{\mathbf s\in\mathbb S^{k-1}}\bm\gamma_2(\mathbf s)\right)>0.
\end{displaymath}
respectively. Then, the corresponding density functions $g_{1}$ and $g_{2}$ exist and are such that
\begin{displaymath}
||g_{1}-g_{2}||_{\infty}\leq \frac{k \Gamma(k/\alpha)}{\alpha (2\pi)^k\underline{\bm\gamma}^{k+1}}
\int_{\mathbb{S}^{k-1}}|(\bm\gamma_{1}(\mathbf{u}))^{1/\alpha}-
(\bm\gamma_{2}(\mathbf{u}))^{1/\alpha}|\mathrm{d}\mathbf{u}.
\end{displaymath}
\end{lem}

\begin{lem}\label{lemma1}
Under \eqref{cont} and \eqref{span}, for every $l=1,\ldots,D$,
\begin{equation}\label{eq:Gammanozero}
\inf_{\mathbf u\in \mathbb{S}^{k-1}}\bm\gamma^{(l)}(\bm u)>0,
\end{equation}
and the distribution of $\fl$ is absolutely continuous with respect to the Lebesgue measure.
\end{lem}

\begin{proof}
If $\inf_{\mathbf u\in \mathbb{S}^{k-1}}\bm\gamma^{(l)}(\bm u)>0$, then absolute continuity of the distribution of $\fl$ follows from Lemma \ref{lemma:nolan}. Since $\mathbf u\mapsto \bm\gamma^{(l)}(\bm u)$ are continuous on $\mathbb{S}^{k-1}$, the minimum is attained. Thus, it is sufficient to show that
\begin{equation}\label{eq:notzero}
\bm\gamma^{(l)}(\bm u)\neq 0\mbox{ for every } \mathbf u\in \mathbb{S}^{k-1}.
\end{equation}
We prove \eqref{eq:notzero} by induction on the NN's layers. If there exists a vector $\mathbf u\in \mathbb{S}^{k-1}$ such that $\bm\gamma^{(1)}(\bm u)=0$, then  it holds that $\mathbf 1^T \mathbf u=0$  and $\mathbf x_i \mathbf u=0$ for every $i$. On the other hand, since $\{\mathbf 1^T, \mathbf x_1,\dots,\mathbf x_I\}$ spans $\mathbb R^k$, then there exists $a_0,\dots,a_I$ such that $\mathbf u^T=a_0\mathbf 1^T+\sum_{i=1}^I a_i\mathbf x_i$. Thus $\mathbf u^T\mathbf u=0$ which contradicts $\mathbf u\in \mathbb S^{k-1}$.
Thus (\ref{eq:Gammanozero}) holds true for $l=1$ and the distribution of $\fu$ is absolutely continuous. Now suppose that (\ref{eq:Gammanozero}) holds true so that the distribution of $\fl$ is absolutely continuous. Since $\phi$ is continuous and strictly increasing, then the distribution of $\phi\circ \fl$ is also absolutely continuous. Thus, for every $\mathbf u\in\mathbb{S}^{k-1}$, $\mathbb{P}[(\phi\circ \fl)\mathbf u=0]=0$, which implies that $\mathbb{E}[|(\phi\circ \fl)\mathbf u|^\alpha]>0$. Thus, $\bm\gamma^{(l+1)}(\mathbf u)>0$ for every $\mathbf u\in\mathbb{S}^{k-1}$.
\end{proof}

By Lemma \ref{lemma1}, $\mathbb{P}[f_{i}^{(l)}(\mathbf{X})\in\cdot\;]$ is absolutely continuous with respect to the Lebesgue measure, for $l=1,\dots, D$, and we denote by $g^{(l)}$ its density function. The next three lemmas deal with the distribution of $f_{i}^{(l)}(\mathbf{X},n)$.

\begin{lem}\label{lemma2}
Under \eqref{cont} and \eqref{span}, for every $l=2,\ldots,D$ and every $\mathbf u\in\mathbb{S}^{k-1}$, as $n\rightarrow+\infty$
\begin{displaymath}
\bm\gamma_n^{(l)}(\bm u)\stackrel{\text{a.s.}}{\longrightarrow} \bm\gamma^{(l)}(\bm u).
\end{displaymath}
\end{lem}

\begin{proof}
Under \eqref{cont} the assumptions of Theorem \ref{teo1} hold true. Its proof shows that, for every $l$, the conditional distribution of $f_i^{(l)}(\mathbf X,n)$, given $\mathcal G_{n,l-1}$
is $St_k(\alpha,\Gamma_n^{(l)})$ and converges a.s. in the weak topology, to the law of $f_i^{(l)}(\mathbf X)$, which is $St_k(\alpha,\Gamma^{(l)})$. As a consequence, by ii) and iii) above, for every $\mathbf u\in \mathbb S^{k-1}$,
\begin{displaymath}
(\bm\gamma_n^{(l)}(\mathbf u))^{1/\alpha}\stackrel{\text{a.s.}}{\longrightarrow}
(\bm\gamma^{(l)}(\mathbf u))^{1/\alpha}.
\end{displaymath}
Since the function $x\rightarrow x^{\alpha}$ is continuous, the thesis follows.
\end{proof}

\begin{lem}\label{lemma3}
For every $l=2,\ldots,D$ and every  $\mathbf u\in \mathbb{S}^{k-1}$
\begin{displaymath}
\liminf_{n\rightarrow+\infty}\bm\gamma_n^{(l)}(\bm u)>\frac{\bm\gamma^{(l)}(\bm u)}{2}\quad a.s.
\end{displaymath}
\end{lem}

\begin{proof}
The proof is a direct consequence of Lemma \ref{lemma1} and Lemma \ref{lemma2}.
\end{proof}

\begin{lem} \label{lemma4}
Let $\underline{\bm\gamma}^{(l)}=\inf_{\mathbf u\in \mathbb{S}^{k-1}}\bm\gamma^{(l)}(\bm u)$. Then, under \eqref{cont} and \eqref{span}, for every $l=2,\ldots,D$ it holds that
\begin{displaymath}
\liminf_n\inf_{\mathbf u\in \mathbb{S}^{k-1}}\bm\gamma_n^{(l)}(\bm u)>\frac{\underline{\bm\gamma}^{(l)}}{4}\quad a.s.
\end{displaymath}
\end{lem}
\begin{proof}
We start by defining $r=\min(1,\underline{\bm\gamma}^{(l)}/(4\max(\alpha,1)\sigma_w^\alpha \bar{\phi}^\alpha k^{\alpha/2}))$. Since $\mathbb{S}^{k-1}$ is compact, then there exist $m\in \mathbb N$ and $\mathbf u_1,\dots,\mathbf u_m$ in $\mathbb{S}^{k-1}$ such that: for every $\mathbf u\in \mathbb S^{k-1}$ there exists $h\in\{1,\dots,m\}$ such that $||\mathbf u-\mathbf u_h||^{\min(\alpha,1)}<r$. Now, let $N_h=\{\omega\in\Omega\text{ : }\liminf_n \bm\gamma_n^{(l)}(\mathbf u_h)\leq\bm\gamma^{(l)}(\mathbf u_h)/2\}$ for $h=1,\ldots,m$. Then, $\mathbb{P}[\cup_hN_h]=0$ and for every $\omega\in (\cup_h N_h)^c$ it holds that
\begin{displaymath}
\liminf_n\bm\gamma^{(l)}_n(\mathbf u_h)(\omega)>\frac{\underline{\bm\gamma}^{(l)}}{2}
\end{displaymath}
for every $h=1,\ldots,m$. Now, fix $\omega\in (\cup_hN_h)^c$ and let $\mathbf u\in \mathbb{S}^{k-1}$. Moreover, let $h$ be such that $||\mathbf u-\mathbf u_h||^{\min(\alpha,1)}<r$.
Then, if $\alpha\leq 1$, for every $n$, we can write
\begin{align*}
|\bm\gamma_n^{(l)}(\bm u)(\omega)-
\bm\gamma^{(l)}_n(\mathbf u_h)(\omega)| & \leq\left|\frac{\sigma_w^\alpha}{n}\sum_{j=1}^n\left(|  (\phi\circ f_j^{(l-1)}(\mathbf X,n))\mathbf u |^\alpha
-|  (\phi\circ f_j^{(l-1)}(\mathbf X,n))\mathbf u_h |^\alpha\right)\right|                                                                                                                                                                     \\
                                        & \leq \frac{\sigma_w^\alpha}{n}\sum_{j=1}^n |(\phi\circ  f_j^{(l-1)}(\mathbf X,n)(\mathbf u-\mathbf u_h )|^\alpha\leq \sigma_w^\alpha  \bar \phi^\alpha k^{\alpha/2} ||\mathbf u-\mathbf u_h||^\alpha \\
                                        & <\frac{\underline{\bm\gamma}^{(l)}}{4}.
\end{align*}
On the other hand, if $\alpha>1$ then by the Lagrange theorem there exists $\mathbf u(h,n)=\theta(h,n)\mathbf u+(1-\theta(h,n))\mathbf u_h$,  with $\theta(h,n)\in [0,1]$, such that we can write
\begin{align*}
 & |\bm\gamma_n^{(l)}(\bm u)(\omega)-\bm\gamma^{(l)}_n(\mathbf u_h)(\omega)|                                                            \\
 & \quad\leq\left|\frac{\sigma_w^\alpha}{n}\sum_{j=1}^n\left(|  (\phi\circ f_j^{(l-1)}(\mathbf X,n))\mathbf u |^\alpha
-|  (\phi\circ f_j^{(l-1)}(\mathbf X,n))\mathbf u_h |^\alpha\right)\right|                                                              \\
 & \quad\leq\frac{\sigma_w^\alpha}{n}\sum_{j=1}^n
\alpha |  (\phi\circ f_j^{(l-1)}(\mathbf X,n))\mathbf u(h,n) |^{\alpha-1} |(\phi\circ f_j^{(l-1)}(\mathbf X,n))(\mathbf u-\mathbf u_h)| \\
 & \quad\leq\frac{\sigma_w^\alpha}{n}\sum_{j=1}^n
\alpha ||  \phi\circ f_j^{(l-1)}(\mathbf X,n)||^\alpha ||\mathbf u-\mathbf u_h||\leq \sigma_w^\alpha \alpha \bar{\phi}^\alpha k^{\alpha/2}||\mathbf u-\mathbf u_h|| < \frac{\underline{\bm\gamma}^{(l)}}{4}.
\end{align*}
Thus
\begin{displaymath}
\liminf_n\bm\gamma^{(l)}(\bm u)\geq \liminf \bm\gamma^{(l)}_n( \mathbf{u}_h)-\frac{\underline{\bm\gamma}^{(l)}}{4}>\frac{\underline{\bm\gamma}^{(l)}}{4}.
\end{displaymath}
\end{proof}

We define the set $N=\{\omega\in\Omega\text{ : } \liminf_n\inf_{\mathbf u\in \mathbb{S}^{k-1}}\bm\gamma_n^{(l)}(\mathbf u)\leq \underline{\bm\gamma}^{(l)}/4\}$. Then, it holds that $\mathbb{P}[N]=0$ and for every $\omega\in N^c$ there exists $n_0=n_0(\omega)$ such that
\begin{displaymath}
\inf_{\mathbf u\in \mathbb{S}^{k-1}}\bm\gamma^{(l)}_n(\mathbf u)(\omega)>\underline{\bm\gamma}^{(l)}/4\quad \mbox{for every }n\geq n_0.
\end{displaymath}
From Lemma \ref{lemma:nolan}, for every $\omega\in N^c$, $\mathbb{P}_{n,l-1}[f_i^{(l)}(\mathbf X,n)\in\cdot\;](\omega)$ is absolutely continuous with respect to the Lebesgue measure, for $n\geq n_0(\omega)$. We denote by $g^{(l)}(n)$ a version of the (random) density function of $\mathbb{P}_{n,l-1}[f_i^{(l)}(\mathbf X,n)\in\cdot\;]$, with respect to the Lebesgue measure. We can extend the definition of $g^{(l)}(n)$ to every $n$ and every $\omega$. The next theorem establishes the convergence rate of $(f^{(l)}_{i}(\mathbf{X},n))_{i\geq1}$ to $(f^{(l)}_{i}(\mathbf{X}))_{i\geq1}$.

\begin{theorem}[The ``joint growth" setting]\label{rate_joint}
Under the assumptions \eqref{cont} and \eqref{span}, we denote by $g^{(l)}(n)$ and $g^{(l)}$ the versions of the density functions of the distributions $\mathbb{P}_{n,l-1}[f_i^{(l)}(\mathbf X,n)\in\cdot\;\;]$ and $\mathbb P[f_{j}(\mathbf{X})\in \cdot\;\;]$, respectively, with respect to the Lebesgue measure. If $\delta_{2}<1/2$ and for every $l=3,\ldots,D$
\begin{equation}\label{eq:cond_delta}
\delta_l<\frac{\delta_{l-1}}{1+2k/\alpha}.
\end{equation}
then, as $n\rightarrow+\infty$,
\begin{displaymath}
n^{\delta_l}||g^{(l)}(n)-g^{(l)}||_\infty\stackrel{p}{\longrightarrow}0.
\end{displaymath}
\end{theorem}
\begin{proof}
We restrict to the set $N^{c}=\{\omega\in\Omega\text{ : } \liminf_n\inf_{\mathbf u\in \mathbb{S}^{k-1}}\bm\gamma_n^{(l)}(\bm u)> \underline{\bm\gamma}^{(l)}/4\}$. Similarly to the proof of Theorem \ref{teo1}, we consider induction over the NN's layers $l=1,\ldots,D$. In particular, by Lemma \ref{lemma:nolan} it holds that
\begin{displaymath}
n^{\delta_l}||g^{(l)}(n)-g^{(l)}||_\infty
\leq c_1 n^{\delta_l} \int_{\mathbb{S}^{k-1}}|(\bm\gamma^{(l)}_n(\mathbf u))^{1/\alpha}-(\bm\gamma^{(l)}(\mathbf u))^{1/\alpha}|\mathrm{d}\mathbf u,
\end{displaymath}
with $c_1=k\Gamma(k/\alpha)/(\alpha (2\pi)^k\underline{\bm\gamma}^{k+1})$. Since for $n$ large enough, the measures $\bm\gamma^{(l)}_n(\mathbf u)d\bm u$ and $\bm\gamma^{(l)}(\mathbf u)d\bm u$ are bounded by \eqref{eq:gammabound} and bounded away from zero by Lemma \ref{lemma1} and Lemma \ref{lemma4}, proving that $n^{\delta_l} \int_{\mathbb{S}^{k-1}}|(\bm\gamma^{(l)}_n(\mathbf u))^{1/\alpha}-(\bm\gamma^{(l)}(\mathbf u))^{1/\alpha}|\mathrm{d}\mathbf u$ converges in probability to zero is equivalent to proving that
$n^{\delta_l} \int_{\mathbb{S}^{k-1}}|\bm\gamma^{(l)}_n(\mathbf u)-\bm\gamma^{(l)}(\mathbf u)|\mathrm{d}\mathbf u$ converges in probability to zero. In particular, for $l=2$, we can write
\begin{align*}
 & n^{\delta_2} \int_{\mathbb{S}^{k-1}}|\bm\gamma^{(2)}_n(\mathbf u)-\bm\gamma^{(2)}(\mathbf u)|\mathrm{d}\mathbf u\leq \int_{\mathbb{S}^{k-1}}\biggl|
\frac{1}{n^{1-\delta_2}}\sum_{j=1}^n\left(|(\phi\circ f_j^{(1)}(\mathbf X))\mathbf u|^\alpha-\mathbb{E}[|(\phi\circ f_j^{(1)}(\mathbf X))\mathbf u|^\alpha]\right)
\biggr|\mathrm{d}\mathbf u.
\end{align*}
By means of \cite[Theorem 2]{von(65)},
\begin{align*}
 & \mathbb{E}\left[
\int_{\mathbb{S}^{k-1}}
\biggl|\frac{1}{n^{1-\delta_2}}\sum_{j=1}^n\left(|(\phi\circ f_j^{(1)}(\mathbf X))\mathbf u|^\alpha- \mathbb{E}[| (\phi\circ f_j^{(1)}(\mathbf X)) \mathbf u|^\alpha]\right)\biggr|^2\mathrm{d}\mathbf u\right] \\
 & \quad=
\int_{\mathbb{S}^{k-1}}\mathbb{E}\left[
\biggl|\frac{1}{n^{1-\delta_2}}\sum_{j=1}^n\left(| (\phi\circ f_j^{(1)}(\mathbf X))\mathbf u|^\alpha-\mathbb{E}[|(\phi\circ f_j^{(1)}(\mathbf X))
\mathbf u|^\alpha]\right)\biggr|^2\right]\mathrm{d}\mathbf u                                                                                                                                                    \\
 & \quad\leq 2 \frac{1}{n^{1-2\delta_2}}
\int_{\mathbb{S}^{k-1}}
\mathbb{E}\left[|(\phi\circ f_j^{(1)}(\mathbf X))\mathbf u|^{2\alpha}\right] \mathrm{d}\mathbf u                                                                                                                \\
 & \quad\leq \frac{2k^\alpha\bar{\phi}^{2\alpha}\lambda_{k-1}(\mathbb S^{k-1})}{n^{1-2\delta_2}}\rightarrow 0
\end{align*}
as $n\rightarrow+\infty$. Now, recall that the convergence in $L^2$ in $(\Omega\times\mathbb{S}^{k-1},\mathcal{G}\otimes\mathcal B(\mathbb{S}^{k-1}),\mathbb{P}\times\lambda_{k-1})$, with $\mathcal{B}(\mathbb{S}^{k-1})$ being the Borel sigma algebra of $\mathbb{S}^{k-1}$, implies the convergence in $L^1$. Therefore, we can write
\begin{displaymath}
\mathbb{E}\left[
\int_{\mathbb{S}^{k-1}}
\biggl|\frac{1}{n^{1-\delta_2}}\sum_{j=1}^n|(\phi\circ f_j^{(l)}(\mathbf X))\mathbf u |^\alpha- \mathbb{E}[|  (\phi\circ f_j^{(1)}(\mathbf X)) \mathbf u|^\alpha]\biggr|\mathrm{d}\mathbf u\right]\rightarrow 0
\end{displaymath}
as $n\rightarrow+\infty$, which implies
\begin{equation}\label{eq:Th1conv2}
\int_{\mathbb{S}^{k-1}}\biggl|\frac{1}{n^{1-\delta_2}}\sum_{j=1}^{n}|  (\phi\circ f_j^{(1)}(\mathbf X))\mathbf u |^\alpha- \mathbb{E}[|  (\phi\circ f_j^{(1)}(\mathbf X)) \mathbf u|^\alpha]\biggr|\mathrm{d}\mathbf u\stackrel{p}{\longrightarrow}0
\end{equation}
as $n\rightarrow+\infty$. This completes the proof for $l=2$. Now, as the main induction hypothesis, we assume that
\begin{displaymath}
n^{\delta_{l-1}}||g^{(l-1)}(n)-g^{(l-1)}||_\infty\stackrel{p}{\longrightarrow} 0
\end{displaymath}
as $n\rightarrow+\infty$, for some $l\geq 3$. By (\ref{eq:cond_delta}), there exists $\delta$ such that $0<\delta<\frac{\alpha}{2k}(\delta_{l-1}-\delta_l)-\delta_l$. Then, we write
\begin{align}\label{eq:deco}
 & n^{\delta_l}\int_{\mathbb S^{k-1}}|\bm\gamma_n^{(l)}(\mathbf u)-\bm\gamma^{(l)}(\mathbf u)|\mathrm{d}\mathbf u                                                                                                                                                 \\
 & \notag\quad\leq n^{\delta_l} \sigma_w^\alpha\int_{\mathbb{S}^{k-1}}	\biggl|\frac{1}{n}\sum_{j=1}^n|(\phi\circ f_j^{(l-1)}(\mathbf X,n))\mathbf u|^\alpha-\mathbb{E}_{n,l-2}[|(\phi\circ f_j^{(l-1)}(\mathbf X,n))\mathbf u|^{\alpha}]\biggr|\mathrm{d}\mathbf u \\
 & \notag\quad\quad+ n^{\delta_l}\sigma_w^\alpha\int_{\mathbb{S}^{k-1}} \biggl|\mathbb{E}_{n,l-2}(|(\phi\circ f_i^{(l-1)}(\mathbf X,n))\mathbf u|^{\alpha})-\mathbb{E}[|(\phi\circ f_i^{(l-1)}(\mathbf X,n))\mathbf u|^{\alpha}]\biggr|\mathrm{d}\mathbf u.
\end{align}
We consider the two terms on the right-hand side of \eqref{eq:deco}. With regards to the first term, by Theorem 2 in \cite{von(65)},
\begin{align*}
 & \mathbb{E}\left[
\int_{\mathbb{S}^{k-1}}
\left|\frac{1}{n^{1-\delta_l}}\sum_{j=1}^n\left(|  (\phi\circ f_j^{(l-1)}(\mathbf X,n))\mathbf u |^\alpha- \mathbb{E}_{n,l-2}[|  (\phi\circ f_j^{(l-1)}(\mathbf X,n)) \mathbf u|^\alpha]\right)\right|^2\mathrm{d}\mathbf u\right]        \\
 & \quad=
\mathbb{E}\left[\int_{\mathbb{S}^{k-1}}\!\!\!\!\!
\mathbb{E}_{n,l-2}\Biggl[
\biggl|\frac{1}{n^{1-\delta_l}}\sum_{j=1}^n\left(| (\phi\circ f_j^{(l-1)}(\mathbf X,n))\mathbf u |^\alpha- \mathbb{E}_{n,l-2}[|  (\phi\circ f_j^{(l-1)}(\mathbf X,n)) \mathbf u|^\alpha]\right)\biggr|^2\Biggr]\mathrm{d}\mathbf u\right] \\
 & \quad\leq 2 \frac{1}{n^{1-2\delta_l}}
\mathbb{E}\left[\int_{\mathbb{S}^{k-1}}
\mathbb{E}_{n,l-2}\left[|(\phi\circ f_i^{(l-1)}(\mathbf X,n))\mathbf u \mid^{2\alpha}\right] \mathrm{d}\mathbf u\right]                                                                                                                   \\
 & \quad\leq \frac{2\bar{\phi}^{2\alpha}k^\alpha\lambda_{k-1}(\mathbb S^{k-1})}{n^{1-2\delta_l}}\rightarrow 0,
\end{align*}
as $n\rightarrow+\infty$. Since convergence in $L^2$ in $(\Omega\times\mathbb{S}^{k-1},\mathcal{G}\otimes\mathcal B(\mathbb{S}^{k-1}),\mathbb{P}\times \lambda_{k-1})$ implies the convergence in $L^1$,
\begin{displaymath}
\mathbb{E}\left[
\int_{\mathbb{S}^{k-1}}
\biggl|\frac{1}{n^{1-\delta_l}}\sum_{j=1}^{n}| (\phi\circ f_j^{(l-1)}(\mathbf X,n))\mathbf u|^\alpha- \mathbb{E}_{n,l-2}[|  (\phi\circ f_j^{(l-1)}(\mathbf X,n)) \mathbf u|^\alpha]\biggr|\mathrm{d}\mathbf u\right]\rightarrow 0,
\end{displaymath}
as $n\rightarrow+\infty$,  which implies
\begin{equation}\label{eq:Th1conv2}
\int_{\mathbb{S}^{k-1}}
\biggl|\frac{1}{n^{1-\delta_l}}\sum_{j=1}^{n}|(\phi\circ f_j^{(l-1)}(\mathbf X,n))\mathbf u |^\alpha- \mathbb{E}_{n,l-2}[|  (\phi\circ f_j^{(l-1)}(\mathbf X,n)) \mathbf u|^\alpha]\biggr
|\mathrm{d}\mathbf u\stackrel{p}{\longrightarrow}0
\end{equation}
as $n\rightarrow+\infty$. For the second term on the right-hand side of \eqref{eq:deco}, by Lemma \ref{lemma:byc} with $R=R^{(l)}(n)=n^{2(\delta_l+\delta)/\alpha}$,
\begin{align*}
 & n^{\delta_l}\int_{\mathbb{S}^{k-1}} |\mathbb{E}_{n,l-2}[|(\phi\circ f_i^{(l-1)}(\mathbf X,n))\mathbf u|^{\alpha}]-\mathbb{E}[|(\phi\circ f_i^{(l-1)}(\mathbf X,n))\mathbf u|^{\alpha}]|\mathrm{d}\mathbf u                          \\
 & \quad\leq n^{\delta_l}\int_{\mathbb{S}^{k-1}} \left|\int_{\mathbb R^k}|(\phi\circ f)\mathbf u|^\alpha g^{(l-1)}(f,n)\mathrm{d}f-\int_{\mathbb R^k}|(\phi\circ f)\mathbf u|^\alpha g^{(l-1)}(f)\mathrm{d}f\right|\mathrm{d}\mathbf u \\
 & \quad\leq n^{\delta_l}\int_{\mathbb{S}^{k-1}} \int_{||f||\leq R^{(l)}(n)}|(\phi\circ f)\mathbf u|^\alpha |g^{(l-1)}(f,n)-g^{(l-1)}(f)|\mathrm{d}f\mathrm{d}\mathbf u                                                                \\
 & \quad\quad+ n^{\delta_l}\int_{\mathbb{S}^{k-1}} \int_{||f||>R^{(l)}(n)}|(\phi\circ f)\mathbf u|^\alpha g^{(l-1)}(f,n)\mathrm{d}f\mathrm{d}\mathbf u                                                                                 \\
 & \quad\quad\quad+n^{\delta_l}\int_{\mathbb{S}^{k-1}} \int_{||f||>R^{(l)}(n)}|(\phi\circ f)\mathbf u|^\alpha g^{(l-1)}(f)\mathrm{d}f\mathrm{d}\mathbf u                                                                               \\
 & \quad\leq n^{\delta_l}\int_{\mathbb{S}^{k-1}}\int_{||f||\leq R^{(l)}(n)}\bar{\phi}^\alpha k^{\alpha/2} ||g^{(l-1)}(n)-g^{(l-1)}||_\infty \mathrm{d}f \mathrm{d}\mathbf u                                                            \\
 & \quad\quad+ n^{\delta_l}\int_{\mathbb{S}^{k-1}}\int_{||f|| > R^{(l)}(n)}\bar{\phi}^\alpha k^{\alpha/2} g^{(l-1)}(f)\mathrm{d}f \mathrm{d}\mathbf u                                                                                  \\
 & \quad\quad\quad+ n^{\delta_l}\int_{\mathbb{S}^{k-1}}\int_{||f||> R^{(l)}(n)}\bar{\phi}^\alpha k^{\alpha/2} g^{(l-1)}(f,n)\mathrm{d}f\mathrm{d}\mathbf u                                                                             \\
 & \quad\leq n^{\delta_l} \lambda_{k-1}(\mathbb{S}^{k-1})                                                                                                                                                                              \\
 & \quad\quad\times
\left(\frac{\pi^{k/2}}{\Gamma(k/2+1)}(R^{(l)}(n))^k\bar{\phi}^\alpha k^{\alpha/2} ||g^{(l-1)}(n)-g^{(l-1)}||_\infty  +2\bar{\phi}^\alpha k^{\alpha/2}\frac{c(\alpha)k\bar{\bm\gamma}^{1/\alpha}}{(R^{(l-1)}(n))^{\alpha/2}}\right)     \\
 & \quad\leq n^{\delta_l}\lambda_{k-1}(\mathbb{S}^{k-1}) \bar\phi^\alpha k^{\alpha/2}                                                                                                                                                  \\
 & \quad\quad \times \left(n^{2k(\delta_l+\delta)/\alpha-\delta_{l-1}}\frac{\pi^{k/2}}{\Gamma(k/2+1)}
n^{\delta_{l-1}}||g^{(l-1)}(n)-g^{(l-1)}||_\infty+
\frac{\bar{\bm\gamma} c(\alpha)k}{n^{\delta_l+\delta}}\right)\stackrel{p}{\longrightarrow} 0
\end{align*}
as $n\rightarrow+\infty$, since we have $2k(\delta_l+\delta)/\alpha-\delta_{l-1}+\delta_l<0$. This completes the proof.
\end{proof}

Theorem \ref{rate_joint} provides a refinement of Theorem \ref{teo1} by establishing the sup-norm convergence rate of a deep Stable NNs to a Stable SPs, in the ``joint growth" setting. As in Theorem \ref{teo1}, the proof in the ``joint growth" setting requires the use of an induction argument or induction hypothesis over the NN's layers $l=1,\ldots,D$. In particular, by means of such an induction argument, it is shown how the convergence rate $n^{\delta_{l}}$ is affected by the depth of the layer, i.e. $l$, and by the dimension $k\geq1$ of the input That is,
\begin{equation}\label{lay2}
\delta_{2}<\frac{1}{2}
\end{equation}
and for $l=3,\ldots,D$
\begin{equation}\label{layl}
\delta_l<\frac{\delta_{l-1}}{1+2k/\alpha}.
\end{equation}
Then, according to \eqref{lay2} and \eqref{layl}, the assumption of the ``sequential growth" setting implies two critical effects on the convergence rates: i) for any $l=3,\ldots,D$, the deeper the layer $l$ in the NN the slower the convergence rate; ii) for each fixed $l=3,\ldots,D$, the larger the dimension $k\geq1$ of the inuput the slower the convergence rate. Such a behaviour is completely determined by the assumption of the ``joint growth" setting, and a different behaviour is expected under the assumption of the ``sequential growth" setting.

\subsection{The ``sequential growth" setting}

The ``sequential growth" setting consists in assuming that, for any $l=1,\ldots,D$, the width $n\rightarrow+\infty$ one layer at a time. To deal with such a setting, we consider the deep Stable NN $(\check f_i^{(l)}(\mathbf X,n))_{i\geq1}$ defined as follows
\begin{displaymath}
\check f_i^{(1)}(\mathbf X)=\sum_{j=1}^I w_{i,j}^{(1)}\mathbf x_j+b_i^{(1)}\mathbf 1^T
\end{displaymath}
and
\begin{displaymath}
\check f_i^{(l)}(\mathbf X,n)=\frac{1}{n^{1/\alpha}}\sum_{j=1}^n w_{i,j}^{(l)}(\phi\circ \check f_j^{(l-1)}(\mathbf X))+b_i^{(l)}\mathbf 1^T,
\end{displaymath}
where $ (\check f_j^{(l-1)}(\mathbf X))_{j\geq1}$ is a sequence of $k$-dimensional (row) random vectors such that, as $n\rightarrow+\infty$, it holds that
\begin{displaymath}
(\check f_i^{(l-1)}(\mathbf X,n))_{i\geq1}\stackrel{w}{\longrightarrow}(\check f_i^{(l-1)}(\mathbf X))_{i\geq1}.
\end{displaymath}
The distribution of $(\check{f}_i^{(l-1)}(\mathbf X))_{i\geq1}$ coincides with the distribution of the Stable SP $(f_i^{(l-1)}(\mathbf X))_{i\geq1}$ in Theorem \ref{teo1}. Under the ``sequential growth" setting, the study of convergence rates of $(f_{i}^{(l)}(\mathbf{X},n))_{i\geq1}$ to $(f_{i}^{(l)}(\mathbf{X}))_{i\geq1}$ reduces to the study of convergence rates of $(\check f_i^{(l)}(\mathbf X,n))_{i\geq1}$, which is a simpler problem.

Let $\mathcal{G}_{l}$ denote the sigma algebra generated by $\{(\check f_i^{(l^{\prime})}(\mathbf X))_{i\geq1}\text{ : }l^{\prime}\leq l\}$, for any $l\geq1$, and let $\mathcal{G}_{0}$ denote the trivial sigma algebra. To establish sup-norm convergence rates, it is useful to consider linear transformations of the random vectors $\check{f}^{(l)}_{i}(\mathbf{X},n)$ and $\check{f}^{(l)}_{i}(\mathbf{X})$. If $\mathbf{u}$ is a $k$-dimensional (column) vector then
\begin{itemize}
\item[i)] $\check{f}^{(1)}_{i}(\mathbf{X})\mathbf u\sim\text{St}(\alpha,(\check{\bm\gamma}^{(1)}(\mathbf{u}))^{1/\alpha})$, where
\begin{displaymath}
\check{\bm\gamma}^{(1)}(\mathbf{u})=\sigma_{b}^{\alpha}|\mathbf{1}^{T}\mathbf{u}|^{\alpha}+\sigma^{\alpha}_{w}\sum_{j=1}^{I}|\mathbf{x}_{j}\mathbf{u}|^{\alpha};
\end{displaymath}
\item[ii)] $\check{f}^{(l)}_{i}(\mathbf{X},n)\mathbf{u}\,|\,\mathcal{G}_{l-1}\sim\text{St}(\alpha,(\check{\bm\gamma}^{(l)}_{n}(\mathbf{u}))^{1/\alpha})$, where
\begin{displaymath}
\check{\bm\gamma}^{(l)}_{n}(\mathbf{u})=\sigma_{b}^{\alpha}|\mathbf{1}^{T}\mathbf{u}|^{\alpha}+\frac{\sigma^{\alpha}_{w}}{n}\sum_{j=1}^{n}|(\phi\circ \check{f}_{j}^{(l-1)}(\mathbf{X}))\mathbf{u}|^{\alpha}
\end{displaymath}
\item[iii)] $\check{f}^{(l)}_{i}(\mathbf{X})\mathbf{u}\sim\text{St}(\alpha,(\check{\bm\gamma}^{(1)}(\mathbf{u}))^{1/\alpha})$, where
\begin{displaymath}
\check{\bm\gamma}^{(l)}(\mathbf{u})=\sigma^{\alpha}_{b}|\mathbf{1}^{T}\mathbf{u}|^{\alpha}+\sigma^{\alpha}_{w}\mathbb{E}[|(\phi\circ \check{f}_{j}^{(l-1)}(\mathbf{X}))\mathbf{u}|^{\alpha}].
\end{displaymath}
\end{itemize}
Under \eqref{cont},
\begin{equation}\label{eq:gammachecksup}
\max(\check \Gamma_n^{(l)}(\mathbb S^{k-1}),\check \Gamma^{(l)}(\mathbb S^{k-1}))\leq \overline{\bm\gamma}
\end{equation}
for $l=2,\ldots,$D, with $\overline{\bm\gamma}$ defined as in \eqref{eq:gammabound}.
We denote by $\mathbb{E}_{l}$ and $\mathbb{P}_{l}$ the conditional expectation and the conditional distribution, respectively, given $\mathcal{G}_{l}$, and by $\lambda_{k-1}$ the Lebesgue measure on $\mathbb{S}^{k-1}$. Assuming that \eqref{cont} and \eqref{span} hold, along lines similar to that of Lemma \ref{lemma2}, Lemma \ref{lemma3} and Lemma \ref{lemma4} we have
\begin{equation}
\label{eq:gammacheckinf}
\check{\underline{\bm\gamma}}^{(l)}:=\inf_{\mathbf{u}\in\mathbb{S}^{k-1}}\check{\bm\gamma}^{(l)}(\mathbf{u})	>0,
\end{equation}
and
\begin{equation}
\label{eq:gammacheckinf2}
\inf_{\mathbf u\in \mathbb{S}^{k-1}}\check{\bm\gamma}^{(l)}_n(\mathbf u,\omega)>\check{\underline{\bm\gamma}}^{(l)}/4
\end{equation}
for every $l=1,\dots,D$, every $\omega\in \check N^c$ with $\mathbb P[\check N]=0$, and every $n\geq n_0(\omega)$. By Lemma \ref{lemma:nolan}, the probability measure $\mathbb{P}[\check{f}_{i}^{(l)}(\mathbf{X})\in\cdot\;]$ is absolutely continuous with respect to the Lebesgue measure, and we denote by  $\check{g}^{(l)}$ its density function.
Moreover,
for every $\omega\in\check{N}^c$, $\mathbb{P}_{l-1}[\check{f}_i^{(l)}(\mathbf X,n)\in\cdot\;](\omega)$ is absolutely continuous with respect to the Lebesgue measure, for $n\geq n_{0}(\omega)$. In particular, we denote by $\check{g}^{(l)}(n)$ a version of the density function of $\mathbb{P}_{n,l-1}[\check{f}_i^{(l)}(\mathbf X,n)\in\cdot\;]$. We can extend the definition of $\check{g}^{(l)}(n)$ to every $n$ and every $\omega$. The next theorem establishes the convergence rate of $(\check{f}^{(l)}_{i}(\mathbf{X},n))_{i\geq1}$ to $(\check{f}^{(l)}_{i}(\mathbf{X}))_{i\geq1}$.

\begin{theorem}[The ``sequential growth" setting]\label{rate_seq}
Under the assumptions \eqref{cont} and \eqref{span}, we denote by $\check{g}^{(l)}(n)$ and $\check{g}^{(l)}$ the versions of the random density functions of the distributions $\mathbb{P}_{l-1}[\check{f}_i^{(l)}(\mathbf X,n)\in\cdot\;]$ and $\mathbb P[\check{f}_{j}(\mathbf{X})\in \cdot\;]$, respectively, with respect to the Lebesgue measure. For every $l=1,\ldots,D$ and $\epsilon>0$, as $n\rightarrow+\infty$
\begin{displaymath}
n^{1/2-\epsilon}||\check{g}^{(l)}(n)-\check{g}^{(l)}||_\infty\stackrel{p}{\longrightarrow}0.
\end{displaymath}
\end{theorem}
\begin{proof}
We restrict to the set $\check{N}^{c}=\{\omega\in\Omega\text{ : } \liminf_n\inf_{\mathbf u\in \mathbb{S}^{k-1}}\check{\bm\gamma}_n^{(l)}(\mathbf u)(\omega)>\check{\underline{\bm\gamma}}^{(l)}/4\}$. Now, since $\liminf_n\inf_{\mathbf u\in\mathbb{S}^{k-1}}\check{\bm\gamma}_n^{(l)}(\mathbf u)>\check{\underline{\bm\gamma}}^{(l)}/4>0$ and $\inf_{\mathbf u\in\mathbb{S}^{k-1}}\check{\bm\gamma}^{(l)}(\mathbf u)>0$, for $l=1,\ldots,D$, by Lemma \ref{lemma:nolan}, it is sufficient to show that
\begin{equation}\label{eq:th1conv11}
n^{1/2-\epsilon}\int_{\mathbb{S}^{k-1}}|(\check{\bm\gamma}_n^{(l)}(\mathbf u))^{1/\alpha}-(\check{\bm\gamma}^{(l)}(\mathbf u))^{1/\alpha}|\mathrm{d}\mathbf u\stackrel{p}{\longrightarrow} 0,
\end{equation}
as $n\rightarrow+\infty$.
The measures $\check{\bm\gamma}_n^{(l)}(\mathbf u)d\bm u$ and $\check{\bm\gamma}^{(l)}(\mathbf u)d\bm u$  are bounded by \eqref{eq:gammabound}; moreover, they are bounded away from zero, as shown in \eqref{eq:gammacheckinf} and \eqref{eq:gammacheckinf2}. Accordingly, \eqref{eq:th1conv11} is equivalent to the following
\begin{equation}\label{eq:th1conv}
n^{1/2-\epsilon}\int_{\mathbb{S}^{k-1}}|\check{\bm\gamma}_n^{(l)}(\mathbf u)-\check{\bm\gamma}^{(l)}(\mathbf u)|\mathrm{d}\mathbf u\stackrel{p}{\longrightarrow} 0,
\end{equation}
as $n\rightarrow+\infty$. By ii) and iii) above,
\begin{align*}
 & n^{1/2-\epsilon}|\check{\bm\gamma}_n^{(l)}(\mathbf u)-\check{\bm\gamma}^{(l)}(\mathbf u)|= \sigma_w^\alpha\biggl|\frac{1}{n^{1/2+\epsilon}}\sum_{j=1}^n\left(|  (\phi\circ \check f_j^{(l-1)}(\mathbf X))\mathbf u |^\alpha- \mathbb{E}[|  (\phi\circ \check f_j^{(l-1)}(\mathbf X)) \mathbf u|^\alpha]\right)\biggr|.
\end{align*}
and by \cite[Theorem 2]{von(65)},
\begin{align*}
 & \mathbb{E}\left[
\int_{\mathbb{S}^{k-1}}
\biggl|\frac{1}{n^{1/2+\epsilon}}\sum_{j=1}^n\left(|(\phi\circ \check f_j^{(l-1)}(\mathbf X))\mathbf u |^\alpha- \mathbb{E}[|(\phi\circ \check f_j^{(l-1)}(\mathbf X)) \mathbf u|^\alpha]\right)\biggr|^2\mathrm{d}\mathbf u\right]   \\
 & \quad=\int_{\mathbb{S}^{k-1}}\mathbb{E}\left[
\biggl|\frac{1}{n^{1/2+\epsilon}}\sum_{j=1}^n\left(|(\phi\circ \check f_j^{(l-1)}(\mathbf X))\mathbf u|^\alpha- \mathbb{E}[| (\phi\circ \check f_j^{(l-1)}(\mathbf X,n)) \mathbf u|^\alpha]\right)\biggr|^2\right]\mathrm{d}\mathbf u \\
 & \quad\leq 2 \frac{1}{n^{2\epsilon}}
\int_{\mathbb{S}^{k-1}}
\mathbb{E}\left[|(\phi\circ \check f_j^{(l-1)}(\mathbf X))\mathbf u \mid^{2\alpha}\right]\mathrm{d}\mathbf u                                                                                                                          \\
 & \quad \leq \frac{2\bar{\phi}^{2\alpha}k^\alpha \lambda_{k-1}(\mathbb S^{k-1})}{n^{2\epsilon}}\rightarrow 0,
\end{align*}
as $n\rightarrow+\infty$. Now, the convergence in $L^2$ in $(\Omega\times\mathbb{S}^{k-1},\mathcal{G}\otimes\mathcal B(\mathbb{S}^{k-1}),\mathbb{P}\times \lambda_{k-1})$, with $\mathcal{B}(\mathbb{S}^{k-1})$ being the Borel sigma algebra of $\mathbb{S}^{k-1}$, implies the convergence in $L^1$. Accordingly, we can write the following
\begin{displaymath}
\mathbb{E}\left[
\int_{\mathbb{S}^{k-1}}\biggl|\frac{1}{n^{1/2+\epsilon}}\sum_{j=1}^{n}|(\phi\circ \check f_j^{(l-1)}(\mathbf X))\mathbf u|^\alpha- \mathbb{E}[|(\phi\circ \check f_j^{(l-1)}(\mathbf X)) \mathbf u|^\alpha]\biggr|\mathrm{d}\mathbf u\right]\rightarrow 0,
\end{displaymath}
as $n\rightarrow+\infty$,
which implies that
\begin{displaymath}
\int_{\mathbb{S}^{k-1}}
\biggl|\frac{1}{n^{1/2+\epsilon}}\sum_{j=1}^{n}|(\phi\circ \check f_j^{(l-1)}(\mathbf X))\mathbf u |^\alpha- \mathbb{E}[| (\phi\circ \check f_j^{(l-1)}(\mathbf X)) \mathbf u|^\alpha]\biggr|\mathrm{d}\mathbf u\stackrel{p}{\longrightarrow}0,
\end{displaymath}
as $n\rightarrow+\infty$. This completes the proof.
\end{proof}

Theorem \ref{rate_seq} provides an interesting complement to Theorem \ref{rate_joint}, as it highlights a critical difference between the ``joint growth" and the ``sequential growth" settings. Differently from the ``joint growth" setting, the ``sequential growth" setting does not require the use of an induction argument over the NN’s layers $l=1,\ldots,D$, as in the study of the convergence rate at layer $l$ it is assumed that the layer $l-1$ has already reached its limit. In particular, Theorem \ref{rate_seq} shows how in the ``sequential growth" setting the convergence rate $n^{\delta_{l}}$ is not affected by the the depth of the layer, i.e. $l$, or the dimension $k\geq1$ of the input. That is,
\begin{equation}\label{layseq}
\delta_{l}<\frac{1}{2}
\end{equation}
for every $l=1,\ldots,D$. According to \eqref{layseq}, the assumption of the ``sequential growth" setting implies that the convergence rate is constant with respect to the depth of the layer $l=1,\ldots,D$ and the dimension $k\geq1$ of the input. While at the level of the infinitely wide limit in Theorem \ref{teo1} there are no difference between the ``joint growth" and the ``sequential growth" settings, as both settings lead to the same infinitely wide Stable process, our results show how a difference between these settings appears at the refined level of convergence rate. To the best of our knowledge, our work is the first to provide a quantitative result on the difference between the ``joint growth" setting and the ``sequential growth" setting.


\section{Discussion}\label{sec4}

We discuss the potential of our results with respect to to Bayesian inference, gradient descent via neural tangent kernels and depth limits, and we present open challenges in large-width asymptotics for deep Stable NNs.

\subsection{Bayesian inference}\label{sec:bi}

From Theorem \ref{teo0}, we know that infinitely wide deep Gaussian NNs give rise to i.i.d. centered Gaussian SPs at every layer $1 \leq l \leq D$. Now, assume that we are given a training dataset of $k_{tr}$ distinct observations $((\mathbf{x}_1,y_1),\ldots,(\mathbf{x}_{k_{tr}},y_{k_{tr}}))$ where each $\mathbf{x}_j \in \mathbb{R}^I$ is a $I$-dimensional input and $y_i$ its corresponding scalar output. Then, we are interested in determining the conditional distributions of the limiting Gaussian SPs over a set of $k_{te}$ test input values given the training dataset, that is the distribution of
\begin{equation}\label{eq:gp_cond}
(f_i^{(l)}(\mathbf{x}_{k_{tr}+1}),\dots,f_i^{(l)}(\mathbf{x}_k)) \mid (f_i^{(l)}(\mathbf{x}_{1})=y_1,\dots,f_i^{(l)}(\mathbf{x}_{k_{tr}})=y_{k_{tr}}),
\end{equation}
where $k=k_{tr}+k_{te}$ and we indexed training observations from $1$ to $k_{tr}$, test inputs from $k_{tr}+1$ to $k$. Theorem \ref{teo0} establishes that the covariance matrices of the limiting Gaussian SPs over all $k$ inputs, one for each layer $1 \leq l \leq D$, can be computed via a recursion over such layers. \cite{Lee(18)} proposes an efficient quadrature solution that keeps the computational requirements manageable for an arbitrary activation $\phi$. Once the covariance matrix over all $k$ inputs for a given layer $l$ is available, standard results on multivariate Guassian vectors establish that the distribution of (\ref{eq:gp_cond}) is multivariate Gaussian, whose mean vector and covariance matrix is obtainable via simple (but potentially costly) algebraic manipulations \citep{rasmussen2006gaussian}.

In the context of deep Stable NNs, computing the distribution of \eqref{eq:gp_cond} is a more challenging task with respect to deep Gaussian NNs. Note that it is possible to approximately simulate from the distribution of $(f_i^{(l)}(\mathbf{x}_1),\dots,f_i^{(l)}(\mathbf{x}_k))$. In particular, since $\Gamma^{(1)}$ is a discrete measure then exact simulations algorithms are available with a computational cost of $\mathcal{O}(IK)$ per sample \citep{nolan2008overview,Sam(94)}. Therefore, we generate $M$ samples $\widetilde{f}^{(1)}_i$, $i=1,\dots,M$, in $\mathcal{O}(MIK)$, and use these to approximate $f^{(2)} \sim \text{St}_k(\alpha,\Gamma^{(2)})$ with $\text{St}_k(\alpha,\widetilde{\Gamma}^{(2)})$ where
\begin{displaymath}
\widetilde{\Gamma}^{(2)} = \sigma_b^\alpha ||1||^\alpha \zeta_{\frac{\boldsymbol{1}}{||\boldsymbol{1}||}} + \sigma_w^\alpha \frac{1}{M}\sum_{j=1}^M ||\phi(\widetilde{f}^{(1)}_j)||^\alpha\zeta_{\frac{\phi(\widetilde{f}^{(1)}_j)}{||\phi(\widetilde{f}^{(1)}_j)||}}.
\end{displaymath}
We can repeat this procedure by generating (approximate) random samples $\widetilde{f}^{(2)}_j$, with a cost of $\mathcal{O}(M^2k)$, that in turn are used to approximate $\Gamma^{(3)}$ and so on. The sequential discretization of the spectral measure to perform approximate sampling is not advantageous. Such a procedure can be shown to be equivalent (in distribution) to sequentially sampling over the layers of the finite NN of width $n=M$. In any case, we still have the problem of computing a statistic of \eqref{eq:gp_cond} or sampling from it, to perform prediction. In general, performing inference with the Stable SPs of Theorem \ref{teo1} remains an open problem.

\subsection{Neural tangent kernel}

In Section \ref{sec:bi} we reviewed how the interplay between deep Gaussian NNs and Gaussian SPs allows to perform Bayesian inference on the infinitely wide SP. This corresponds to a ``weakly-trained" regime, in the sense that the posterior mean predictions of \eqref{eq:gp_cond} are equivalently obtained by assuming a quadratic loss function, and then fitting only the final linear layer of the NN with gradient flow, i.e. gradient descent with infinitesimal learning rate \citep{Aro(19)}. This result thus establishes an equivalence between a specific training setting for deep Gaussian NN and a kernel regression. Differently, the works \cite{Jac(18),Lee(19),Aro(19)} consider ``fully-trained" deep Gaussian NNs, in the sense that all the layers are trained jointly, still under the same quadratic loss and gradient flow. It is shown that as the width of the NN goes to infinity, the point predictions are still equivalent to that of a kernel regression, though with respect to a different kernel, which is referred to as the neural tangent kernel. A key assumption in the derivation of the neural tangent kernel is that the gradients are not computed with respect to the standard model parameters, i.e. the weights and biases entering the affine transforms. Instead, they are re-parametrized gradients which are computed with respect to weights distributed as standard Gaussian distributions, with any scaling (standard deviation) applied as a further multiplication. Recently, \cite{favaro2022} introduced an analogous equivalence in the context of ``fully-trained" shallow Stable NNs with a ReLU activation function, showing that the underlying kernel regression is with respect to an $(\alpha/2)$-Stable random kernel. We believe that it would be of interest to study whether the work of \cite{favaro2022} can be extended to the context of deep $\alpha$-Stable NNs with a general activation function, i.e. linear and sub-linear.

\subsection{Depth limits}

In the context of deep Gaussian NNs, information propagation investigates the evolution over depth of the covariance matrix recursion in Theorem \ref{teo0} \citep{Poo(16),Sch(17),Hay(19)}. In particular, following the notation and the assumptions of Theorem \ref{teo0}, it is shown that the $(\sigma_w,\sigma_b)$ positive quadrant is divided into two regions: i) a stable phase; ii) a chaotic phase. Assuming for simplicity $\phi=\tanh$, in the stable phase the limiting Gaussian SP correlation between any two distinct inputs tends to $1$ as the depth grows unbounded, and the limiting Gaussian SP concentrates on constant functions. Under the same assumption $\phi=\tanh$, in the chaotic phase this correlation converges to a random variable, and the limiting Gaussian SP is almost everywhere discontinuous. \cite{Hay(19)} investigates the case where $(\sigma_w,\sigma_b)$ is on the curve separating the stable phase from the chaotic phase, which is typically referred to as the edge of chaos curve. On such a curve, it is shown that the behavior is qualitatively similar to that of the stable phase, but with a lower rate of convergence with respect to depth. Thus, in all cases, the distribution of the limiting Gaussian SP eventually collapse to degenerate and inexpressive distributions as the depth increases.

It would be interesting to investigate the role on the Stable distribution, with $\alpha\in(0,2]$, in the edge of chaos phenomenon. It seems difficult to escape the curse of depth under i.i.d. distributions for the weights, though it might be the case that Stable distributions, with their not-uniformly-vanishing relevance at unit level \cite{Nea(96)}, allow to slow down the rate of convergence to the limiting regime. For deep Gaussian NNs of finite width, a way to avoid the curse of depth is to shrink the distribution of the NN's weights as the total number of layers $D$ increases. This idea has been explored in \cite{cohen2021scaling}, with the critical result that as $D$ goes to infinity the finite-width NN converges to the solution of a stochastic differential equation (SDE). We conjecture that, under appropriate scaling, the same approach applied to a NN whose weights are distributed as Stable distributions would result in converge to the solution of a Levy-driven stochastic differential equation. A more recent line of research focuses on taking joint limits in width and depth \citep{li2021future}. Here, the theory is less developed, and a formal result among the lines of Theorem \ref{teo0} is lacking. However, the partial results that have been obtained so far hint at a class of limiting SPs that might better capture the properties of finitely-sized NNs. Interestingly such limiting SPs are not Gaussian SPs. Therefore, it would be of interest to investigate some extensions of Theorem \ref{teo1} under the more flexible scenario where both the width and depth are allowed to grow, possibly at different rates.

\section*{Acknowledgement}

The authors are grateful to three anonymous Referees for all their comments, corrections, and numerous suggestions that improved remarkably the paper. Stefano Favaro received funding from the European Research Council (ERC) under the European Union's Horizon 2020 research and innovation programme under grant agreement No 817257. Stefano Favaro gratefully acknowledge the financial support from the Italian Ministry of Education, University and Research (MIUR), ``Dipartimenti di Eccellenza" grant 2018-2022.




\begin{thebibliography}{9}

\bibitem[Aitken and Gur-Ari(2020)]{Ait(20)}
\textsc{Aitken, K. and Gur-Ari, G.} (2020). On the asymptotics of wide networks with polynomial activations. \textit{Preprint: arXiv:2006.06687}.

\bibitem[Andreassen and Dyer(2020)]{And(20)}
\textsc{Andreassen, A. and Dyer, E.} (2020). Asymptotics of wide convolutional neural networks. \textit{Preprint: arXiv:2008.08675}.

\bibitem[Antognini(2019)]{Ant(19)}
\textsc{Antognini, J.M.} (2019). Finite size corrections for neural network gaussian processes. \textit{Preprint: arXiv:1908.10030}.

\bibitem[Arora et al.(2019)]{Aro(19)}
\textsc{Arora, S., Du, S.S., Hu, W., Li, Z., Salakhutdinov, R.R. and Wang, R.} (2019). On exact computation with an infinitely wide neural net. In \textit{Advances in Neural Information Processing Systems}.

\bibitem[Basteri and Trevisan(2022)]{Bas(22)}
\textsc{Basteri, A. and Trevisan, D.} (2022). Quantitative Gaussian approximation of randomly initialized deep neural networks \textit{Preprint arXiv:2203.07379}.

\bibitem[Billingsley(1999)]{Bil(99)}
\textsc{Billingsley, P.} (1999). \textit{Convergence of probability measures}. Wiley-Interscience.

\bibitem[Blackwell and Dubins (1962)]{blackwell}
\textsc{Blackwell, D. and Dubins, L.} (1962) {Merging of opinions with Increasing Information}. \textit{The Annals of Mathematical Statistics} \textbf{33}, 882 -- 886.

\bibitem[Blum et al.(1958)]{Blu(58)}
\textsc{Blum, J.R., Chernoff, H., Rosenblatt, M. and Teicher, H.} (1958). Central limit theorems for interchangeable processes. \textit{Canadian Journal of Mathematics} \textbf{10}, 222-229.

\bibitem[Bordino et al.(2022)]{Bor(22)}
\textsc{Bordino, A., Favaro, S. and Fortini} (2022). Infinite-wide limits for Stable deep neural networks: sub-linear, linear and super-linear activation functions. \textit{Preprint available upon request.}

\bibitem[Byczkowski et al.(1993)]{Byc(93)}
\textsc{Byczkowski, T., Nolan, J.P. and Rajput, B.} (1993). Approximation of multidimensional Stable densities. \textit{Journal of Multivariate Analysis} \textbf{46}, 13--31.

\bibitem[Cohen et al.(2021)]{cohen2021scaling}
\textsc{Cohen, A., Cont, R., Rossier, A. and Xu, R.} (2021). Scaling properties of deep residual networks. In \textit{International Conference on Machine Learning}.

\bibitem[Der and Lee(2006)]{Der(06)}
\textsc{Der, R. and Lee, D.} (2006). Beyond Gaussian processes: on the distributions of infinite networks. In \textit{Advances in Neural Information Processing Systems}.

\bibitem[Eldan et al.(2021)]{Eld(21)}
\textsc{Eldan, R., Mikulincer, D. and Schramm, T.} (2021). Non-asymptotic approximations of neural networks by Gaussian processes. In \textit{Conference on Learning Theory}.

\bibitem[Favaro et al.(2020)]{favaro2020}
\textsc{Favaro, S., Fortini, S. and Peluchetti, S.} (2020). Stable behaviour of infinitely wide deep neural networks. In \textit{International Conference on Artificial Intelligence and Statistics}.

\bibitem[Favaro et al.(2022)]{favaro2022}
\textsc{Favaro, S., Fortini, S. and Peluchetti, S.} (2022). Neural tangent kernel analysis of shallow $\alpha$-Stable ReLU neural networks. \textit{Preprint arXiv:2206.08065}.

\bibitem[Fortuin et al.(2019)]{For(20)}
\textsc{Fortuin, V., Garriga-Alonso, A., Wenzel, F., Ratsch, G, Turner, R.E., van der Wilk, M. and Aitchison, L.} (2020). Bayesian neural network priors revisited. In \textit{Advances in Neural Information Processing Systems}.

\bibitem[Garriga-Alonso et al.(2018)]{Gar(18)}
\textsc{Garriga-Alonso, A., Rasmussen, C.E. and  Aitchison, L.} (2018). Deep convolutional networks as shallow Gaussian processes. In \textit{International Conference on Learning Representation}.

\bibitem[Gnedenko and Kolmogorov(1954)]{Gne(54)}
\textsc{Gnedenko, B.V. and Kolmogorov, A.N.} (1954). \textit{Limit distributions for sums of independent random variables}. Addison-Wesley.

\bibitem[Hayou et al.(2019)]{Hay(19)}
\textsc{Hayou, S. and Doucet, A. and Rousseau, J.} (2019). On the impact of the activation function on deep neural networks training. In \textit{International Conference on Machine Learning}.

\bibitem[Hazan and Jaakkola(2015)]{Haz(15)}
\textsc{Hazan, T. and Jaakkola, T.} (2015). Steps toward deep kernel methods from infinite neural networks. \textit{Preprint: arXiv:1508.05133}.

\bibitem[Hodgkinson and Mahoney(2021)]{Hod(21)}
\textsc{Hodgkinson, L. and Mahoney, M.} (2021). Multiplicative noise and heavy tails in stochastic optimization. In \textit{International Conference on Machine Learning}.

\bibitem[Jacot et al.(2018)]{Jac(18)}
\textsc{Jacot, A., Gabriel, F, and Hongler, C.} (2018). Neural tangent kernel: convergence and generalization in neural networks. In \textit{Advances in Neural Information Processing Systems}.

\bibitem[Joe and Kuo(2008)]{joe2008notes}
\textsc{Joe, S. and Kuo, F.Y.} (2008). Notes on generating Sobol sequences. \textit{ACM Transactions on Mathematical Software} \textbf{29}, 49--57.

\bibitem[Klukowski(2021)]{Klu(21)}
\textsc{Klukowski, A.} (2021). Rate of convergence of polynomial networks to Gaussian processes \textit{Preprint arXiv: 2111.03175}.

\bibitem[LeCun et al(2015)]{Lec(15)}
\textsc{LeCun, Y., Bengio, Y. and Hinton, G.} (2015). Deep learning. \textit{Nature} \textbf{521}, 436--444.

\bibitem[Lee et al.(2020)]{Lee(20)}
\textsc{Lee, J., Schoenholz, S. Pennington, J., Adlam, B., Xiao, L., Novak, R. and Sohl-Dickstein, J.} (2020). Finite versus infinite neural networks: an empirical study. In \textit{Advances in Neural Information Processing Systems}.

\bibitem[Lee et al.(2018)]{Lee(18)}
\textsc{Lee, J., Sohldickstein, J.,  Pennington, J., Novak, R., Schoenholz, S. and Bahri, Y.} (2018). Deep neural networks as Gaussian processes. In \textit{International Conference on Learning Representation}.

\bibitem[Lee et al.(2019)]{Lee(19)}
\textsc{Lee, J., Xiao, L., Schoenholz, S., Bahri, Y., Sohl-Dickstein, J. and Pennington, J.} (2019). Wide neural networks of any depth evolve as linear models under gradient descent. In \textit{Advances in Neural Information Processing Systems}.

\bibitem[Li et al.(2021b)]{li2021bayesian}
\textsc{Li, C., Dunlop, M. and Stadler, G.} (2021). Bayesian neural network priors for edge-preserving inversion. \textit{Preprint arXiv:2112.10663}.

\bibitem[Li et al.(2021)]{li2021future}
\textsc{Li, M.B., Nica, M. and Roy, D.M.} (2021). The future is log-Gaussian: ResNets and their infinite-depth-and-width limit at initialization. \textit{Preprint arXiv:2106.04013}.

\bibitem[Matthews et al.(2018)]{Mat(18)}
\textsc{Matthews, A.G., Rowland, M., Hron, J., Turner, R.E. and Ghahramani, Z.} (2018). Gaussian process behaviour in wide deep neural networks. In \textit{International Conference on Learning Representations}.

\bibitem[Neal(1996)]{Nea(96)}
\textsc{Neal, R.M.} (1996). \textit{Bayesian learning for neural networks}. Springer.

\bibitem[Nolan(2008)]{nolan2008overview}
\textsc{Nolan, J.P.} (2010). An overview of multivariate Stable distributions. \textit{Preprint, Department of Mathematics and Statistics at American University}.

\bibitem[Nolan(2010)]{Nol(10)}
\textsc{Nolan, J.P.} (2010). Metrics for multivariate Stable distributions. \textit{Banach Center Publications} \textbf{90}, 83--102.

\bibitem[Novak et al.(2018)]{Nov(18)}
\textsc{Novak, R., Xiao, L., Bahri, Y., Lee, J., Yang, G., Hron, J., Abolafia, D., Pennington, J. and Sohldickstein, J.} (2018). Bayesian deep convolutional networks with many channels are Gaussian processes. In \textit{International Conference on Learning Representation}.

\bibitem[Poole et al.(2016)]{Poo(16)}
\textsc{Poole, B., Lahiri, S., Raghu, M., Sohl-Dickstein, J. and Ganguli, S.} (2016). Exponential expressivity in deep neural networks through transient chaos. In \textit{Advances in Neural Information Processing Systems}.

\bibitem[Rasmussen and Williams(2006)]{rasmussen2006gaussian}
\textsc{Rasmussen, C.E. and Williams, C.K.I.} (2006). \textit{Gaussian Processes for Machine Learning.} MIT Press.

\bibitem[Samoradnitsky and Taqqu(1994)]{Sam(94)}
\textsc{Samoradnitsky, G. and Taqqu, M.S} (1994). \textit{Stable non-Gaussian random processes: stochastic models with infinite variance}. Chapman and Hall/CRC.

\bibitem[Schoenholz et al.(2017)]{Sch(17)}
\textsc{Schoenholz, S., Gilmer, J., Ganguli, S. and Sohl-Dickstein, J.} (2017). Deep information propagation. In \textit{International Conference on Learning Representation}.

\bibitem[von Bahr and Esseen(1965)]{von(65)}
\textsc{von Bahr, B. and Esseen, C.} (1965). Inequalities for the $r$th absolute moment of a sum of random variables. \textit{Annals of Mathematical Statistics} \textbf{1}, 299--303.

\bibitem[Williams(1997)]{Wil(97)}
\textsc{Williams, C.K.} (1997). Computing with infinite networks.. In \textit{Advances in Neural Information Processing Systems}.

\bibitem[Yang(2019)]{Yan(19)}
\textsc{Yang, G.} (2019). Scaling limits of wide neural networks with weight sharing: Gaussian process behavior, gradient independence, and neural tangent kernel derivation. \textit{Preprint: arXiv:1902.04760}.

\bibitem[Yang(2019a)]{Yan(19a)}
\textsc{Yang, G.} (2019). Tensor programs I: wide feedforward or recurrent neural networks of any architecture are Gaussian processes. \textit{Preprint: arXiv:1910.12478}.

\end{thebibliography}
\end{document}